\documentclass[11pt]{article}
\usepackage[margin=1in]{geometry}
\usepackage[T1]{fontenc}
\usepackage{lmodern,microtype}
\usepackage{amsmath,amssymb,amsthm,mathtools}
\usepackage{booktabs,array,tabularx,graphicx}
\usepackage{float}
\usepackage{cite}
\usepackage{etoolbox}
\newcommand{\citep}[1]{\cite{#1}}
\newcommand{\citet}[1]{\cite{#1}}
\apptocmd{\thebibliography}{\setlength{\itemsep}{0pt}\setlength{\parsep}{0pt}}{}{}
\usepackage{xcolor}
\usepackage[colorlinks=true,linkcolor=blue!55!black,citecolor=blue!55!black,urlcolor=blue!55!black]{hyperref}
\usepackage{enumitem}
\hypersetup{pdftitle={Variance-Aware Non-Stationary Bandits: All-Budget Regret Bounds and Adaptation Costs}}
\setlist{topsep=3pt,itemsep=2pt}
\floatstyle{ruled}
\newfloat{algorithm}{tbp}{loa}
\floatname{algorithm}{Algorithm}
\newcounter{pseudoline}
\newlength{\pseudoindent}
\newenvironment{algorithmic}[1][1]{%
  \setlength{\pseudoindent}{0pt}%
  \begin{list}{\arabic{pseudoline}.}{\usecounter{pseudoline}%
  \setlength{\leftmargin}{2em}\setlength{\labelsep}{0.5em}%
  \setlength{\itemsep}{2pt}\setlength{\parsep}{0pt}}}{\end{list}}
\newcommand{\State}{\item\hspace*{\pseudoindent}}
\newcommand{\Require}{\item[]\textbf{Input:} }
\newcommand{\For}[1]{\State\textbf{for} #1 \textbf{do}\addtolength{\pseudoindent}{1em}}
\newcommand{\EndFor}{\addtolength{\pseudoindent}{-1em}\State\textbf{end for}}
\newcommand{\If}[1]{\State\textbf{if} #1 \textbf{then}\addtolength{\pseudoindent}{1em}}
\newcommand{\Else}{\addtolength{\pseudoindent}{-1em}\State\textbf{else}\addtolength{\pseudoindent}{1em}}
\newcommand{\EndIf}{\addtolength{\pseudoindent}{-1em}\State\textbf{end if}}
\newcommand{\Comment}[1]{\hfill\textit{#1}}
\newtheorem{theorem}{Theorem}[section]
\newtheorem{lemma}[theorem]{Lemma}
\newtheorem{proposition}[theorem]{Proposition}
\newtheorem{corollary}[theorem]{Corollary}

\newtheorem{remark}[theorem]{Remark}

\newcommand{\E}{\mathbb E}
\newcommand{\Prb}{\mathbb P}
\newcommand{\ind}{\mathbf 1}
\newcommand{\Reg}{R}
\newcommand{\Rreal}{\widetilde R}
\newcommand{\Var}{\operatorname{Var}}
\newcommand{\wh}[1]{\widehat{#1}}

\newcommand{\cF}{\mathcal F}
\newcommand{\cP}{\mathcal P}
\newcommand{\softO}{\widetilde O}

\title{Variance-Aware Non-Stationary Bandits:\\All-Budget Regret Bounds and Adaptation Costs}

\author{Xiaojin Zhang\footnote{School of Computer Science and Technology, Huazhong University of Science and Technology, China, email: xiaojinzhang@hust.edu.cn}}
\hypersetup{pdfauthor={Xiaojin Zhang}}

\date{}

\begin{document}
\hypersetup{pageanchor=false}
\begin{titlepage}
\maketitle
\thispagestyle{empty}
\begin{abstract}
Low reward variance reduces estimation cost, but does not eliminate
monitoring of arms that may improve while unobserved. We refine the
variance--variation tradeoff and characterize
how both costs depend on supplied variance and variation budgets in
bounded, non-stationary multi-armed bandits. Let $A$ be the number of
arms, $T$ the horizon, $B$ a bound on consecutive mean variation
excluding the initial vector, and $Q$ a bound on the sum of the largest
per-round reward variances. For every admissible pair $Q,B$, the minimax
expected dynamic regret is, within a single logarithmic factor,
\[
 \min\{T,\;A+\sqrt{AQ}+(AQBT)^{1/3}
                +\sqrt{AT}\min\{B,\sqrt B\}\}.
\]
The deterministic boundary $Q=0$ is characterized without logarithms.
For small positive budgets, a safe calibration phase identifies a
common interval containing all future means of surviving arms. Clipping
observations to a wider interval reduces the cost of restarting a prior
hybrid static-regret algorithm; an original-mean analytical reference
controls its residuals without charging the variation of clipped means.
Balancing clipping bias and monitoring closes the small-budget noisy
regime. Fixed oblivious constructions give the matching lower bound.
We also quantify adaptation costs: on a disclosed deterministic
one-improvement family, a stationary-regret budget has a matching
worst-case change-regret frontier, including a finite improvement pulse.
This implies an $\Omega((T/A)^{1/4})$ loss when one policy competes with
all supplied-budget deterministic oracle rates. The budget-free
low-noise scheduling problem remains separate.
\end{abstract}
\end{titlepage}
\hypersetup{pageanchor=true}

\section{Introduction}
\label{sec:intro}
Non-stationary bandit learning combines two distinct problems. Samples
estimate an arm's current reward, while continued monitoring discovers
changes in arms that have not been sampled recently. Reward variance
governs the first problem. It does not, by itself, govern the second.
Even with deterministic rewards, an initially inferior arm can become
optimal without changing the reward of the incumbent.

We ask how reward-noise variance and cumulative mean variation
jointly determine minimax difficulty, including vanishing-budget
regimes. We distinguish learning with supplied budgets from adapting
across classes whose budgets are unknown, and retain explicit
dependence on the number of arms and the horizon.
Variance-sensitive dynamic regret and noise-free monitoring already
appear in \citet{wei2016tracking}. Our main refinement concerns
vanishing mean variation with positive observation variance; the
ordinary restart bound is a consequence of \citet{ito2021hybrid}.

\paragraph{The noisy minimax landscape.}
For supplied envelope variance $Q$ and variation $B$, define
\begin{equation}
 \Xi(A,T;Q,B)=\min\{T,A+\sqrt{AQ}+(AQBT)^{1/3}
                    +\sqrt{AT}\min\{B,\sqrt B\}\}.
 \label{eq:introXi}
\end{equation}
Theorem~\ref{thm:jointminimax} gives matching lower and upper bounds,
within one factor $\log(AT)$, for \emph{every} admissible pair of budgets.
The costs are initialization, stationary estimation, statistical drift,
and deterministic monitoring. For $Q=0$, Theorem~\ref{thm:detminimax}
removes all logarithms. The lower constructions are fixed before
interaction and retain the arm dependence and budget endpoints.

The small-budget noisy upper bound needs more than ordinary restarting.
A variance-sensitive calibration phase costs only
$O(A\ell+\sqrt{AQ\ell})$, where $\ell=O(\log(AT))$, even when
it takes many rounds. It safely
removes arms that cannot become optimal and certifies an interval
containing the future means of all survivors. Clipping rewards to a
wider interval of width $O(w)$ then replaces the restart cost $A$ by
$O(wA)$, while contributing bias at most $O(Q/w)$.
We balance $\sqrt{AwBT}$ and $Q/w$ by taking
$w=\max\{B,(Q^2/(ABT))^{1/3}\}$ when this is below one.
The post-calibration learner is the existing hybrid algorithm of
\citet{ito2021hybrid}. Proposition~\ref{prop:itoreduction} records
what plain restarting already implies; Section~\ref{sec:smallbudget}
provides the calibration and clipping argument needed for the sharper
small-budget rate. No new static base algorithm is claimed.

\paragraph{The deterministic boundary.}
Let $\Delta$ be the sum of consecutive mean changes, excluding the
initial mean vector. If rewards are deterministic and a budget
$B\ge\Delta$ is supplied, Theorem~\ref{thm:detminimax} gives the
minimax rate
\begin{equation}
 \Psi(A,T,B)=\min\{T,\;A+\sqrt{AT}\min\{B,\sqrt B\}\}.
 \label{eq:Psi}
\end{equation}
For $B\le1$, this is $\min\{T,A+B\sqrt{AT}\}$; for $B\ge1$,
the drift contribution has the familiar $\sqrt{ABT}$ scale. The upper
bound exploits information supplied by the small budget. After one
initial observation per arm, an arm far below the largest observed
reward cannot become optimal during the remaining horizon. All arms
left by this safe filter have gap $O(B)$, so subsequent monitoring
costs $O(B)$ per probe rather than $1$.

\paragraph{The price of adapting.}
An algorithm that is not promised stationarity cannot simply commit
after initial observations. On a family with one hidden improvement,
Theorem~\ref{thm:frontier} identifies a matching frontier within this family between a
stationary expected-regret budget $r$ and worst-case one-change regret.
Both sides of the frontier are proved. The upper-bound policy spends
its monitoring budget through independent probes and commits after
observing the uniquely recognizable improvement. The lower bound
couples each alternative with a stationary trajectory until its first
informative observation, and sums the resulting detection delays.
More precisely, for a disclosed gap $g$, $n=A-1$ and $T\ge16A$,
the optimal change regret under stationary budget $r$ is
\begin{equation}
 \Theta\!\left(\min\left\{gT,\;
 g\sqrt{nT}+\frac{g^2nT}{r+gn}\right\}\right).
 \label{eq:introfrontier}
\end{equation}
The same frontier covers a single finite improvement pulse.
Corollary~\ref{cor:competitive} further gives an
$\Omega((T/A)^{1/4})$ worst-case ratio to the supplied-budget oracle.
These statements concern explicit classes and side information.

\paragraph{Proof roadmap.}
Section~\ref{sec:deterministic} gives a safe filter and a
logarithm-free deterministic policy. Section~\ref{sec:noisyminimax}
proves the noisy upper bound by calibration and clipped restarts;
its last-pull accounting bounds calibration regret without a harmonic
sum. Section~\ref{sec:barriers} proves the adaptation frontier.
The full oblivious lower constructions and an independent
high-probability empirical-variance analysis appear in the appendices.

\paragraph{Relation to prior work.}
Variance-sensitive dynamic regret is already studied in
\citep{wei2016tracking,wang2025variance}; inverse exploration--detection
tradeoffs are classical \citep{garivier2011,garivier2008}.
The textbook treatment in \citet{lattimore2020}, Theorem~31.2,
and the safe-environment results of \citet{suk2025smooth} make clear
that logarithmic stationary regret cannot be extended uniformly by
adding a detector alone.
Prior-free tracking, significant shifts and path-length bounds are
essential comparisons \citep{auer2019,chen2019,wei2021master,suk2022,
abbasi2023,wei2018adaptive,bubeck2019path}.
Section~\ref{sec:related} compares their actual model and theorem
conditions, rather than only displayed rates. In particular, the
small-budget upper rate can also be derived by filtering and rescaling
an earlier deterministic bound, up to logarithmic factors.
The additional small-budget noisy argument controls calibration
regret and clipping bias; the ordinary restart guarantee serves as
its baseline.

\paragraph{Scope of the claims.}
Theorem~\ref{thm:jointminimax} characterizes the expected minimax value
for every \emph{supplied} pair $Q,B$. Its deterministic boundary has
no logarithmic gap, while Theorem~\ref{thm:frontier} characterizes
adaptation on the disclosed one-improvement family. The general
unknown-budget low-noise schedule remains open; it is not a consequence
of either characterization. The separate empirical-variance appendix
provides a high-probability ordinary restart bound, rather than a
high-probability version of the sharp all-budget formula.
Section~\ref{sec:related} attributes the prior algorithmic and
lower-bound ingredients. Full oblivious lower-bound proofs are in
Appendices~\ref{app:detlower} and \ref{app:noisylower}.
\section{Model and quantities}
\label{sec:model}
There are $A\ge2$ arms and $T\ge2$ rounds. Before interaction, an
oblivious environment fixes reward laws $\nu_{t,a}$ on $[0,1]$, with
means $\mu_{t,a}$ and variances $\sigma_{t,a}^2$. Reward draws are
independent across rounds. Within-round dependence across arms is
irrelevant because the learner observes only its selected reward.
The policy chooses $a_t$ using past observations and private randomness,
and observes $X_t\sim\nu_{t,a_t}$. In particular,
\[
 \E[X_t\mid\cF_{t-1},a_t]=\mu_{t,a_t},\qquad
 \E[(X_t-\mu_{t,a_t})^2\mid\cF_{t-1},a_t]
 =\sigma_{t,a_t}^2.
\]
For martingale arguments, reveal the current action randomization
before the current reward, so that the action and its conditional
variance are predictable in the refined filtration. When conditioning
at calibration or block boundaries, the history contains only randomness
already used; future algorithm seeds remain independent. No assumption of
identical distributions is made. The results below are stated for this
model; an extension to adaptive environments requires its own precise
filtration and budget convention.

\paragraph{Pseudo-regret and realized regret.}
Let $\mu_t^\star=\max_a\mu_{t,a}$. Define
\begin{align}
 \Reg_T&=\sum_{t=1}^T(\mu_t^\star-\mu_{t,a_t}),
 \label{eq:pseudo}\\
 \Rreal_T&=\sum_{t=1}^T(\mu_t^\star-X_t)
 =\Reg_T+\sum_{t=1}^T(\mu_{t,a_t}-X_t).
 \label{eq:real}
\end{align}
The first quantity is nonnegative and at most $T$. The second also
contains reward noise and can be negative. Our deterministic upper bounds are pathwise, the frontier uses
expected pseudo-regret, and the noisy base bounds hold with high probability. We state the additional martingale term
for realized regret explicitly.

\paragraph{Variation and stationary segments.}
Write
\[
 d_t=\|\mu_t-\mu_{t-1}\|_\infty\quad(t\ge2),\qquad
 \Delta=\sum_{t=2}^T d_t,\qquad
 L=1+\sum_{t=2}^T\ind\{d_t>0\}.
\]
Thus $L$ counts segments, and $L-1$ counts mean changes. For an interval
$I=[s,e]$, let $\Delta_I=\sum_{t=s+1}^e d_t$ and $h_I=e-s+1$.
For a partition into contiguous intervals,
\begin{equation}
 \sum_{I\in\cP}\Delta_I\le\Delta.
 \label{eq:boundary}
\end{equation}
The inequality can be strict: changes at boundaries are excluded from
the internal variation of every interval.

\paragraph{Three different variance quantities.}
We distinguish
\begin{equation}
 \Lambda_T=\sum_{t=1}^T\sigma_{t,a_t}^2,\qquad
 V_T=\sum_{t=1}^T\max_a\sigma_{t,a}^2,\qquad
 W_T=\sum_{t=1}^T\sum_{a=1}^A\sigma_{t,a}^2.
 \label{eq:variances}
\end{equation}
Here $\Lambda_T$ depends on the policy's trajectory; $V_T$ and $W_T$
are deterministic properties of the fixed environment. They satisfy
\[
 0\le\Lambda_T\le V_T\le W_T\le A V_T,
 \qquad V_T\le T/4.
\]
The last bound uses the variance bound for $[0,1]$-valued rewards.
These quantities are not interchangeable when tuning an algorithm or
comparing arm dependence. In particular, one cannot choose a restart
period using the eventual $\Lambda_T$ of the as-yet unrun policy.

Throughout, $\delta\in(0,1)$ and
\begin{equation}
 \ell=\log\frac{64AT^4}{\delta},\qquad
 D_T=\ell(1+\log T).
 \label{eq:logs}
\end{equation}
The notation $\softO$ hides only logarithmic factors in $A,T,1/\delta$.
It never hides powers of $A$.

\section{The deterministic minimax boundary}
\label{sec:deterministic}
Let $\mathcal D(A,T,B)$ contain all deterministic reward arrays in
$[0,1]^{T\times A}$ with $\Delta\le B$. The learner knows $A,T,B$
and the promise of determinism, but not the initial reward vector.
Define
\[
 \mathfrak R_{\rm det}(A,T,B)=
 \inf_{\pi}\sup_{\mu\in\mathcal D(A,T,B)}\E_\pi[R_T].
\]
The infimum allows randomized, nonanticipating policies that may depend
on the supplied budget. The adversary is oblivious.

\begin{theorem}[All-budget deterministic minimax rate]
\label{thm:detminimax}
For every $A,T\ge2$ and $B\in[0,T-1]$, universal constants
$c,C>0$ satisfy
\begin{equation}
 c\Psi(A,T,B)\le\mathfrak R_{\rm det}(A,T,B)
              \le C\Psi(A,T,B).
 \label{eq:detminimax}
\end{equation}
The upper bound is attained by a deterministic policy and holds
pathwise. In particular, the zero-variance minimax drift term is
$B\sqrt{AT}$ for $B\le1$ and $\sqrt{ABT}$ for $B\ge1$, with
initial exploration and the horizon cap retained.
\end{theorem}

\subsection{Safe filtering and the upper bound}
Assume $A<T$; otherwise the bound $R_T\le T$ suffices. Pull every arm
once in the first $A$ rounds, and denote its observed reward by $y_a$.
Let $y_{\max}=\max_a y_a$ and retain
\begin{equation}
 \mathcal C=\{a:y_{\max}-y_a\le2B\},\qquad
 \kappa=\min\{1,4B\}.
 \label{eq:filter}
\end{equation}
Only arms in $\mathcal C$ are subsequently monitored.

\begin{lemma}[The filter is safe]
\label{lem:filter}
Every arm optimal at a round $t>A$ belongs to $\mathcal C$.
Every $a\in\mathcal C$ has future instantaneous regret at most
$\kappa$.
\end{lemma}
\begin{proof}
Since each arm's mean changes by at most $B$ over the entire horizon,
$|y_a-\mu_{t,a}|\le B$. If $a$ is optimal at $t$ and $b$ attained
$y_{\max}$ during initialization, then
$y_{\max}-y_a\le\mu_{t,b}-\mu_{t,a}+2B\le2B$.
For any retained arm,
\[
 \mu_t^\star-\mu_{t,a}
 \le(y_{\max}+B)-(y_a-B)\le4B.
\]
The reward range also bounds this regret by $1$.
\end{proof}

\begin{algorithm}[tb]
\caption{Filtered probe-and-hold (deterministic rewards, known budget)}
\label{alg:filtered}
\begin{algorithmic}[1]
\Require $A,T$, deterministic rewards, supplied variation budget $B$
\If{$A\ge T$}
 \State Pull arbitrary arms for $T$ rounds and stop.
\EndIf
\State In the first $A$ rounds, pull every arm once and store $y_a$.
\State Compute $\mathcal C$ and $\kappa$ from \eqref{eq:filter}.
\If{$B=0$}
 \State Set $H=T$.
\Else
 \State Set $H=\min\{T,\max\{A,
            \lceil\sqrt{\kappa AT/(2B)}\rceil\}\}$.
\EndIf
\For{each remaining consecutive interval of length at most $H$}
 \State Pull each arm in $\mathcal C$ once, or until the interval ends.
 \State In remaining rounds, pull an arm with largest reward in this
 interval's probe observations.
\EndFor
\end{algorithmic}
\end{algorithm}

\begin{proposition}[Filtered monitoring bound]
\label{prop:filtered}
For any $H\in\{A,\ldots,T\}$, this policy satisfies
\begin{equation}
 R_T\le A+\kappa A\left(\frac TH+1\right)+2HB.
 \label{eq:filtered}
\end{equation}
Its choice of $H$ in Algorithm~\ref{alg:filtered}, together with the
trivial horizon bound, proves the upper bound in
Theorem~\ref{thm:detminimax}.
\end{proposition}
\begin{proof}
Initialization costs at most $A$. Each subsequent probe has regret at
most $\kappa$ by Lemma~\ref{lem:filter}. On an interval $I$ with
enough time to probe every retained arm, let $z_a$ be its fresh probe
reward and let $\wh a$ maximize $z_a$. Since every current optimal
arm is retained, for any later round in $I$,
\[
 \mu_t^\star-\mu_{t,\wh a}
 \le\mu_t^\star-z_{a_t^\star}+z_{\wh a}-\mu_{t,\wh a}
 \le2\Delta_I.
\]
If the last interval is too short to finish probing, every round is a
probe and is already charged. Summing, using
$\sum_I\Delta_I\le B$, proves \eqref{eq:filtered}.

For $B=0$, all retained arms have zero future gap and $R_T\le A$.
For $B>0$, write $h_\star=\sqrt{\kappa AT/(2B)}$.
If $A\le h_\star\le T$, then $h_\star\le H\le h_\star+1$;
\eqref{eq:filtered} is bounded by a constant times
$A+\sqrt{\kappa ABT}+B$. Here
$B\le\sqrt{\kappa ABT}$, and
\[
 \sqrt{\kappa ABT}\le2\sqrt{AT}\min\{B,\sqrt B\}.
\]
If $h_\star>T$, setting $H=T$ gives
$2BT<\kappa A$ and \eqref{eq:filtered} is $O(A)$.
For clarity, $\kappa=4B$ only when $B\le1/4$; for
$1/4<B<1$ we use $\kappa=1\le4B$ in the displayed inequality.
If $h_\star<A$ and $B<1/4$, then $\kappa=4B$ and
$h_\star=\sqrt{2AT}<A$, which implies $2T<A$.
This contradicts the standing assumption $A<T$. Hence in the
$h_\star<A$ branch we have $B\ge1/4$, so $\kappa=1$ and
$B>T/(2A)$.
For $B\ge1$, $\sqrt{ABT}>T/\sqrt2$; for $B<1$,
$B\sqrt{AT}>T/2$. The trivial $T$ bound is therefore at most
a constant times \eqref{eq:Psi}. These cases cover rounding and both
clipping boundaries.
\end{proof}

\paragraph{Why this boundary is relevant to variance-aware claims.}
Any full minimax characterization with a zero variance budget must
specialize to \eqref{eq:detminimax}. An oracle upper bound containing
$A+\sqrt{ABT}$ is valid but can be loose for $B\ll1$. Conversely,
the known-budget filter must not be inserted into an algorithm that
does not know a valid variation bound. Its permanent deletions would
then have no certificate.

\section{Minimax bounds with both budgets supplied}
\label{sec:noisyminimax}
Let $\mathcal E(A,T;Q,B)$ be the environments of
Section~\ref{sec:model} satisfying $V_T\le Q$ and $\Delta\le B$.
Define
\[
 \mathfrak R(A,T;Q,B)=\inf_{\pi}\sup_{\nu\in\mathcal E(A,T;Q,B)}
                       \E_\nu[R_T].
\]
Here policies receive $A,T,Q,B$, but no individual reward variances,
means, or change locations. This side information is the same for
every environment in the class. The budgets constrain the entire
environment, rather than only the learner's selected arms.

Write
\begin{align}
 S&=\sqrt{AQ},\qquad C_B=(AQBT)^{1/3},\qquad
 D_B=\sqrt{AT}\min\{B,\sqrt B\},\nonumber\\
 \Xi&=\min\{T,A+S+C_B+D_B\}.
 \label{eq:Xi}
\end{align}
The deterministic theorem explains $D_B$. Statistical hard blocks
supply $S,C_B$; calibration and clipping give the full upper bound.

\begin{theorem}[All-budget variance--variation bounds]
\label{thm:jointminimax}
For every $A,T\ge2$, $Q\in[0,T/4]$ and $B\in[0,T-1]$,
universal constants $c,C>0$ satisfy
\begin{equation}
 c\Xi\le\mathfrak R(A,T;Q,B)
       \le C\log(AT)\,\Xi.
 \label{eq:jointminimax}
\end{equation}
If $Q=0$, the sharper upper bound $C\Xi$ holds without logarithms.
In particular:
\begin{align}
 B\ge1:\quad &\mathfrak R(A,T;Q,B)
 =\widetilde\Theta\!\left(\min\{T,(AQBT)^{1/3}+\sqrt{ABT}\}\right),
 \label{eq:matchedlarge}\\
 B=0:\quad &\mathfrak R(A,T;Q,0)
 =\widetilde\Theta\!\left(\min\{T,A+\sqrt{AQ}\}\right).
 \label{eq:matchedstationary}
\end{align}
Equation~\eqref{eq:matchedstationary} assumes constant means; the
reward laws and their variances may still vary with time.
The upper bound is executable with the supplied budgets and does
not require individual variances, gaps or change locations.
\end{theorem}
For $Q>0$, this is a characterization up to one logarithmic factor,
not a constant-factor characterization. Whether that logarithm can
be removed is not resolved here.

The statistical construction below uses the standard change-of-measure
and independent-block method from the bandit lower-bound literature
\citep{besbes2014,wei2016tracking,wang2025variance}. Its purpose here
is to verify all budget constraints in one convention, including the
arm dependence and the small-budget endpoints. We do not present the
block construction as a new lower-bound principle. The \emph{ordinary}
restart upper bound follows from a prior hybrid static-regret algorithm
(Proposition~\ref{prop:itoreduction}); the small-budget refinement
is proved in Section~\ref{sec:smallbudget}.

\paragraph{A polynomial small-budget separation.}
For fixed $A$, take $Q=1$ and $B=T^{-1/4}$. Theorem~\ref{thm:jointminimax}
gives $\widetilde\Theta(T^{1/4})$ expected regret. The
ordinary unit-range restart guarantee in
Proposition~\ref{prop:itoreduction} gives only
$\widetilde O(T^{3/8})$ after period tuning. Thus the small-budget
refinement can improve the displayed upper guarantee by a polynomial
factor even with a positive variance budget. This comparison concerns
the guarantees; it supplies no lower bound on an unmodified restart
policy's actual regret.

\paragraph{Static guarantee used below.}
For $K\ge2$ arms and $h\ge\max\{K,3\}$ rounds, the hybrid learner of
\citet{ito2021hybrid}, Theorem~3, equation~(6), satisfies, against
every fixed arm $a$,
\begin{equation}
 \E\!\left[\sum_{t=1}^h(\ell_{t,a_t}-\ell_{t,a})\right]
 \le C_0\!\left(K\log(eh)+\sqrt{K\log(eh)\,
             \E[\mathcal Z_h(u)]}\right),
 \label{eq:itohybridexplicit}
\end{equation}
where
\[
 \mathcal Z_h(u)=\sum_{t=1}^h(\ell_{t,a_t}-u_{t,a_t})^2
                  +\sum_{t=1}^{h-1}\|u_{t+1}-u_t\|_1.
\]
The reference $u_t\in[0,1]^K$ is used only in the analysis.
It need not equal the loss expectation and is not an algorithmic
input. The displayed form absorbs a $K\sqrt{\log(eh)}$ term into
$K\log(eh)$. For shorter blocks, we play any eligible arm and use
the trivial regret bound. In particular, we never invoke the cited
algorithm's $(\log h)^{-1}$ parameter on a one- or two-round block.
Proposition~\ref{prop:itoreduction} gives the ordinary restart
reduction explicitly. This separates the cited static theorem from
our calibration and clipping analysis.

\subsection{Safe calibration and clipping for small budgets}
\label{sec:smallbudget}
The ordinary restart bound pays a unit-range initialization cost in every
block. For $B<1$, most arms need not incur this cost repeatedly. We first
identify a fixed set containing every future optimal arm, and a common
interval containing all of its future means. We then clip observations
to a slightly wider interval. Importantly, the hybrid theorem is applied
with the \emph{original} mean sequence as its analytical reference;
we never charge the variation of the clipped means.

\begin{table}[H]
\centering
\caption{The width choice for $0<B<1$. The last column bounds the
additional monitoring and clipping-bias terms; $A+\sqrt{AQ}$ and
$(AQBT)^{1/3}$ are accounted for separately. No clipping is needed in
the last row.}
\label{tab:width}
\begin{tabular}{lll}
\toprule
Envelope budget & Width choice & Additional cost \\
\midrule
$Q\le B^2\sqrt{AT}$ & $w=B$ & $O(B\sqrt{AT})$ \\
$B^2\sqrt{AT}<Q<\sqrt{ABT}$ & $w=(Q^2/(ABT))^{1/3}$ & $O((AQBT)^{1/3})$ \\
$Q\ge\sqrt{ABT}$ & Ordinary restarts & $O((AQBT)^{1/3})$ \\
\bottomrule
\end{tabular}
\end{table}

\begin{samepage}
\begin{lemma}[Calibration with a small regret cost]
\label{lem:calibration}
Suppose $A<T$, $0<B\le w\le1$, $\Delta\le B$ and $V_T\le Q$.
There is an observable calibration phase, ending after $\tau\le T$
rounds, with the following properties. With probability at least
$1-\delta$, its regret is at most
\begin{equation}
 C\bigl(A\ell+\sqrt{AQ\ell}\bigr).
 \label{eq:calibrationcost}
\end{equation}
If calibration finishes before the horizon, it returns a nonempty set
$\mathcal C$ and an observed number $z$ such that, for every
$t>\tau$, every optimal arm belongs to $\mathcal C$, and
\begin{equation}
 \mu_{t,a}\in[z-12w,z+4w]\qquad(a\in\mathcal C).
 \label{eq:calibrationrange}
\end{equation}
The good event and the output certificate are measurable with respect
to the calibration history, for each fixed environment.
\end{lemma}
\end{samepage}
\begin{proof}
Pull each arm once, maintain the empirical statistics and radii
$r_a$ in \eqref{eq:radius} using \emph{all} calibration samples of
that arm, and initially set $\mathcal C=[A]$. Before each subsequent
pull, form
\[
 U_a=\widehat\mu_a+r_a+3B,\qquad
 L_a=\widehat\mu_a-r_a-3B.
\]
Delete each arm with $U_a<\max_{b\in\mathcal C}L_b$.
If all remaining radii are at most $w$, finish calibration and set
$z=\max_{a\in\mathcal C}\widehat\mu_a$. Otherwise pull a remaining
arm with largest radius and update its statistics. Truncate this
procedure if the horizon ends.

Use the simultaneous concentration bounds of
Lemma~\ref{lem:event} only up to the end of calibration.
Lemmas~\ref{lem:confidence} and \ref{lem:radiuscontrol} apply to
any nonanticipating sampling rule; both follow directly from that
concentration lemma, without using the interval-regret theorem.
Since the entire mean path
has variation at most $B$, they give, at any calibration decision time
$s$, for every future $t\ge s$,
\begin{equation}
 |\widehat\mu_a-\mu_{t,a}|\le r_a+3B.
 \label{eq:calibrationconfidence}
\end{equation}
Indeed, the confidence bound for the current mean costs $r_a+2B$,
and moving that mean to round $t$ costs at most another $B$.
A deleted arm is therefore worse than its witnessing arm at every
future round. Deletions preserve every future optimal arm, and at least
one arm always remains.

For a selected arm after initialization, its radius is larger than $w$
and at least that of every remaining arm. For a currently optimal arm
$b$, the fact that the selected arm $a$ survived deletion implies
$U_a\ge L_b$. Together with \eqref{eq:calibrationconfidence}, this gives
\[
 \mu_s^\star-\mu_{s,a}
 \le2r_a+2r_b+12B\le16r_a.
\]
Let $N_a$ be the final calibration count and
$B_a=\sum_{t\le\tau:a_t=a}\sigma_{t,a}^2$.
For $N_a\ge2$, the radius immediately before its last calibration
pull, with $n=N_a-1$, satisfies
\[
 w<r_{a,n}\le
 \frac{32\sqrt{B_a\ell}+280\ell}{n}
       +24B\sqrt{\ell/n}.
\]
If $n\ge2304\ell$, the last term is at most $w/2$, whence
$wn\le64\sqrt{B_a\ell}+560\ell$. For smaller $n$, use
$wn\le2304\ell$. The case $N_a=1$ is immediate. Summing yields
\begin{equation}
 w\tau\le C\bigl(A\ell+\sqrt{AQ\ell}\bigr),
 \label{eq:calibrationtime}
\end{equation}
Indeed, the preceding cases give, for every arm,
$wN_a\le64\sqrt{B_a\ell}+2305\ell$ (also for $N_a=1$),
and, pathwise,
\[
 \sum_aB_a=\sum_{t\le\tau}\sigma_{t,a_t}^2
 \le\Lambda_T\le V_T\le Q.
\]
This proves the time bound independently of
any regret bound. The same reasoning holds when the phase is truncated
before reaching its stopping criterion.

It remains to bound regret without summing a harmonic sequence of
radii. Write $g_{t,a}=\mu_t^\star-\mu_{t,a}$, so that
$|g_{t,a}-g_{s,a}|\le2B$ for every $s,t$. For an arm with
$N_a\ge2$, let $s_a$ be its last calibration pull and put
$n=N_a-1$. Its final pull is a post-initialization selection even if the arm is
deleted later, so the maximal-radius rule gives
$g_{s_a,a}\le16r_{a,n}$. No monotonicity of the radii is needed:
the gap-variation bound compares each earlier pull directly with $s_a$.
Thus all its pulls, including its initial one, contribute at most
\begin{align}
 R_{{\rm cal},a}
 &\le N_a(16r_{a,n}+2B)\nonumber\\
 &\le C\bigl(\sqrt{B_a\ell}+\ell
                 +B\sqrt{N_a\ell}+BN_a\bigr),
 \label{eq:lastpull}
\end{align}
For clarity, $N_a/n\le2$ and \eqref{eq:radiusupper} give
\[
 N_ar_{a,n}\le64\sqrt{B_a\ell}+560\ell
                  +24\sqrt2 B\sqrt{N_a\ell},
\]
which yields the second line of \eqref{eq:lastpull}.
Arms pulled only once cost at most $A\le A\ell$ in total. Summing
\eqref{eq:lastpull}, using Cauchy--Schwarz, gives
\[
 R_{\rm cal}\le C\bigl(A\ell+\sqrt{AQ\ell}
                       +B\sqrt{A\tau\ell}+B\tau\bigr).
\]
Finally, $B\sqrt{A\tau\ell}\le(B\tau+AB\ell)/2$ and
$AB\ell\le A\ell$. Since $B\tau\le w\tau$, the independently
established \eqref{eq:calibrationtime} absorbs both remaining terms. This proves \eqref{eq:calibrationcost}, including the
horizon-truncated case. The last-pull accounting removes the extra
$1+\log T$ that a direct radius sum would introduce.

At termination, all surviving radii are at most $w$. For an arm $b$
attaining $z$, survival gives
$\widehat\mu_a\ge z-r_a-r_b-6B\ge z-8w$.
Equation~\eqref{eq:calibrationconfidence} then implies
\eqref{eq:calibrationrange}. Define the good event by the required
concentration inequalities at all calibration prefixes only. It has
probability at least $1-\delta$ by the simultaneous event, and is
measurable at the stopping time $\tau$; no future rewards are used
in this event or in the certificate.
\end{proof}

\begin{lemma}[Clipped restarts with an original-mean reference]
\label{lem:clipped}
Condition on a calibration history for which every future optimal arm
is retained and \eqref{eq:calibrationrange} holds. Put
$J=[z-13w,z+5w]$, of width
$L=18w$. Run fresh copies of the hybrid algorithm of
\citet{ito2021hybrid} on the surviving arms, in consecutive blocks
of length at most $H$, using normalized losses
\[
 \ell'_{t,a_t}=
 \frac{z+5w-\operatorname{clip}_J(X_t)}{L}.
\]
On a block shorter than $\max\{A,3\}$ rounds, play any surviving arm instead.
If there is only one survivor, play it throughout instead; its original-mean
regret is zero because every future optimal arm is retained.
The conditional expected regret on the remaining rounds is at most
\begin{equation}
 \widetilde O\!\left(
 \sqrt{A\lceil T/H\rceil Q}+wA\lceil T/H\rceil+HB
 \right)+\frac{2Q}{w}.
 \label{eq:clippedrestart}
\end{equation}
No individual variance or mean is an algorithmic input.
\end{lemma}
\begin{proof}
Write $q_t=\max_a\sigma_{t,a}^2$ and
$Y_{t,a}=\operatorname{clip}_J(X_{t,a})$. Every surviving arm's original mean lies
in $J$ at distance at least $w$ from its endpoints. Projection and a
second-moment tail bound give, conditionally on the calibration history,
\begin{equation}
 \E[(Y_{t,a}-\mu_{t,a})^2]\le\sigma_{t,a}^2,
 \qquad |\E[Y_{t,a}]-\mu_{t,a}|\le\sigma_{t,a}^2/w.
 \label{eq:clippingmoments}
\end{equation}
Equivalently, with center $c=z-4w$, the certified mean interval is
$[c-8w,c+8w]$ and $J=[c-9w,c+9w]$. Both sides therefore have
exactly $w$ of guaranteed margin; the apparent asymmetry comes only
from using $z$ as the reference point.
For the second inequality, use
$|Y-X|\le|X-\mu|\ind\{|X-\mu|\ge w\}$ and
$\E[(X-\mu)^2]=\sigma^2$.

On a block $I$ of length $h\ge\max\{A,3\}$, use the analytical reference
\[
 u'_{t,a}=(z+5w-\mu_{t,a})/L\in[0,1].
\]
Theorem~3, equation~(6), of \citet{ito2021hybrid}, rescaled by $L$,
gives clipped static regret at most
\[
 \widetilde O\!\left(
 LA+\sqrt{A\E[\Lambda_I]}+A\sqrt{L\Delta_I}\right).
\]
Here the squared residual is controlled by the first inequality in
\eqref{eq:clippingmoments}, while the reference path satisfies
$\sum_t\|u'_{t+1}-u'_t\|_1\le A\Delta_I/L$.
This invocation is valid even though the clipped observations are
biased relative to the original means: the hybrid theorem permits
\emph{arbitrary} reference sequences and does not require them to be
observation expectations.

Choose as comparator an arm optimal at the start of $I$; it is retained
by hypothesis. Its original
mean regret differs from its clipped-mean regret by at most
$2\sum_{t\in I}q_t/w$; moving from this fixed arm to the dynamic
comparator costs at most $2h\Delta_I$. If $h\ge A$, then
$A\sqrt{L\Delta_I}\le LA+h\Delta_I$.
For a shorter block, $h\le A$, so use directly the range certificate:
$R_I\le16wh\le LA$. Thus each block has conditional expected regret
\[
 \widetilde O\!\left(
 wA+\sqrt{A\E[\Lambda_I]}+h\Delta_I\right)
 +\frac2w\sum_{t\in I}q_t.
\]
Apply the same argument conditionally at every block start, using a
fresh algorithm seed. The actual number of post-calibration blocks is
$M_\tau=\lceil(T-\tau)/H\rceil\le\lceil T/H\rceil$.
Cauchy--Schwarz over these blocks gives the variance contribution
$\sqrt{AM_\tau Q}$; their restart costs sum to $wAM_\tau$.
Their drift costs sum to at most $HB$, and the clipping bias to $2Q/w$.
Appendix~\ref{app:stopping} details the conditional sum and tail budgets,
including a short last block. This proves \eqref{eq:clippedrestart}. Independence of
future reward laws from the calibration history follows from the
oblivious independent-round model. Conditioning here is on a good
\emph{prefix}, not on a concentration event involving future rewards.
\end{proof}

\begin{algorithm}[tb]
\caption{Safe calibration followed by clipped restarts}
\label{alg:clipped}
\begin{algorithmic}[1]
\Require $A<T$, $0<B<1$, $Q>0$, $(Q^2/(ABT))^{1/3}<1$
\State Set $w,H$ by \eqref{eq:cliptuning}; set $\delta=T^{-4}$.
\State Set $\mathcal C=[A]$ and pull every arm once.
\For{$t=A+1,\ldots,T$}
 \State Compute all calibration radii from \eqref{eq:radius}.
 \State Delete arms with $\widehat\mu_a+r_a+3B
 <\max_{b\in\mathcal C}(\widehat\mu_b-r_b-3B)$.
 \If{$\max_{a\in\mathcal C}r_a\le w$}
  \State Set $\tau=t-1$, $z=\max_{a\in\mathcal C}\widehat\mu_a$,
  $J=[z-13w,z+5w]$; finish calibration.
  \State Break the loop and keep $\mathcal C,J$ fixed.
 \Else
  \State Pull an arm of largest radius in $\mathcal C$ and update
  its empirical statistics.
 \EndIf
\EndFor
\If{calibration finished with $\tau<T$}
 \If{$|\mathcal C|=1$}
  \State Play its only arm for all remaining rounds.
 \Else
  \State Divide the remaining rounds into consecutive blocks of at most $H$ rounds.
  \State In a block of length $h\ge\max\{A,3\}$, run a fresh copy of
  \citet{ito2021hybrid}, Algorithm~1, with horizon $h$; in a shorter
  block, play any surviving arm.
  \State Supply each learner only the selected-arm loss
  $(z+5w-\operatorname{clip}_J(X_t))/(18w)$.
 \EndIf
\EndIf
\end{algorithmic}
\end{algorithm}

\subsubsection{Completing the upper bound}
\begin{proof}[Upper bound in Theorem~\ref{thm:jointminimax}]
If $A\ge T$, use the trivial horizon bound. If $Q=0$, use
Theorem~\ref{thm:detminimax}. If $B=0$, a single unmodified hybrid
run gives $\widetilde O(A+\sqrt{AQ})$ by
Proposition~\ref{prop:itoreduction}. For $B\ge1$, that proposition
and oracle tuning already give the result.

It remains to consider $0<B<1$ and $Q>0$. Put
\[
 w_0=\left(\frac{Q^2}{ABT}\right)^{1/3}.
\]
If $w_0\ge1$, then $Q^2\ge ABT$, so
$\sqrt{ABT}\le(AQBT)^{1/3}$. The ordinary restart bound is again
sufficient. Otherwise set
\begin{equation}
 w=\max\{B,w_0\}<1,\qquad
 H=\min\left\{T,\left\lceil\max\left\{
 1,(AQT/B^2)^{1/3},\sqrt{AwT/B}\right\}\right\rceil\right\}.
 \label{eq:cliptuning}
\end{equation}
Run calibration with $\delta=T^{-4}$; if it ends early, run the
clipped restart policy on its surviving arms. Let $G\in\mathcal H_\tau$
be the good calibration-prefix event, and let $K_{\rm cal},K_{\rm post}$
be the uniform bounds in Lemmas~\ref{lem:calibration} and \ref{lem:clipped}.
The tower property gives
\[
 \E[R_T]\le K_{\rm cal}+K_{\rm post}+T\Prb(G^c).
\]
Appendix~\ref{app:stopping} proves this step explicitly, including
$\tau=T$ and independence of future randomness.
We give the period calculation directly for the clipped bound.
Write
\[
 \begin{gathered}
 h_v=(AQT/B^2)^{1/3},\quad h_m=\sqrt{AwT/B},\quad
 s=\max\{1,h_v,h_m\},\\
 D=(AQBT)^{1/3},\qquad E=\sqrt{AwBT}.
 \end{gathered}
\]
Let $M=\lceil T/H\rceil$ and
$F(H)=\sqrt{AMQ}+wAM+BH$. We claim
\begin{equation}
 F(H)\le C_1\bigl(A+\sqrt{AQ}+D+E\bigr)
 \label{eq:clippedperiodbound}
\end{equation}
with, for example, $C_1=2$. The balancing identities used below are
\[
 Bh_v=D,\qquad \sqrt{AQT/h_v}=D,\qquad
 Bh_m=E,\qquad AwT/h_m=E.
\]
If $s\le T$, then $s\le H=\lceil s\rceil\le s+1$,
including $s=T$, because $T$ is an integer. Using $M\le1+T/H$ gives
\begin{align*}
 \sqrt{AMQ}&\le\sqrt{AQ}+\sqrt{AQT/H}
                 \le\sqrt{AQ}+D,\\
 wAM&\le Aw+AwT/H\le A+E,\\
 BH&\le B(s+1)\le2B+Bh_v+Bh_m=2B+D+E.
\end{align*}
Here $H\ge h_v,h_m$, $w<1$, and $2B\le A$ since $B<1$ and
$A\ge2$. This also covers $s=1$, integer periods, and a short last
block, whose restart cost has already been included in $M$. Adding the
three displayed estimates gives
$F(H)\le2A+\sqrt{AQ}+2D+2E$, which proves the claimed constant.
If $s>T$, the horizon cap gives $H=T$ and $M=1$, so
$F(T)=\sqrt{AQ}+Aw+BT$. At least one of $h_v>T$ or $h_m>T$
holds. Indeed, $h_v>T$ implies $AQ>B^2T^2$, whereas
$h_m>T$ implies $Aw>BT$. Thus the first gives
$BT<\sqrt{AQ}$, and the second gives $BT<Aw\le A$.
In either case $F(T)\le2(A+\sqrt{AQ})$, so \eqref{eq:clippedperiodbound} also holds when
either or both unconstrained periods exceed the horizon.
No property of an untruncated restart learner is used in this calculation.

Set $\mathcal L=\log(eAT)$. At $\delta=T^{-4}$ we have
$\ell=\log(64AT^8)\le8\mathcal L$ and
$\log(eT)\le\mathcal L$. The explicit stopped-expectation bound in
Appendix~\ref{app:stopping} and \eqref{eq:clippedperiodbound} give
\begin{align}
 \E[R_T]
 &\le C_{\rm cal}(A\ell+\sqrt{AQ\ell})
       +C'\log(eT)F(H)+2Q/w+T^{-3}\nonumber\\
 &\le C_2\mathcal L(A+\sqrt{AQ}+D+E)+2Q/w+T^{-3},
 \label{eq:prebalanced}
\end{align}
where $C_{\rm cal}$ is a universal calibration constant and
$C_2=8C_{\rm cal}+2C'$ is sufficient. These are additive phase
costs; their logarithmic factors do not multiply.
If $w=B$, then $Q\le B^2\sqrt{AT}$, whence
$Q/w\le B\sqrt{AT}$ and $\sqrt{AwBT}=B\sqrt{AT}$.
If $w=w_0\ge B$, both $Q/w$ and $\sqrt{AwBT}$ equal
$(AQBT)^{1/3}$. This proves the small-budget bound.
On $G^c$, total regret is at most $T$ and $\Prb(G^c)\le T^{-4}$,
giving the displayed $T^{-3}$ term. All branches use supplied
information and observed rewards only.
Thus \eqref{eq:prebalanced} is at most
$C_3\mathcal L(A+\sqrt{AQ}+D+B\sqrt{AT})$ for a universal $C_3$,
after increasing it to absorb $T^{-3}$. Combining this with
$R_T\le T$ and $\mathcal L\le2\log(AT)$ proves the stated single
logarithm for $0<B<1$. The other noisy branches use the same
explicit hybrid bound directly; the separate high-probability SEV-UCB
analysis is not used to establish this expected single-logarithm claim.
\end{proof}

For $B\ge1$ and $A\le T$, both $A$ and $\sqrt{AQ}$ are at most
a constant times $\sqrt{ABT}$; for $A>T$, the horizon cap applies.
This proves \eqref{eq:matchedlarge}. Setting $B=0$ gives
\eqref{eq:matchedstationary}. For $Q=0$, the deterministic theorem
matches the entire budget range without logarithms.

\section{Adaptation on a disclosed deterministic family}
\label{sec:barriers}
The deterministic minimax theorem permits different policies for
different supplied budgets. We next quantify the limitation on a single
policy required to perform well on a stationary instance and on
one-change alternatives.

Fix $g\in(0,1/2]$. The stationary baseline $\nu^0_g$ has rewards
$1/2$ on arm $1$ and $1/2-g$ on every challenger. For $a\ge2$
and $\tau\in\{2,\ldots,T\}$, the alternative $\nu^{a,\tau}_g$
changes only arm $a$ to $1/2+g$ from $\tau$ onward. This family
has zero noise variance, and each alternative has exactly one change
with $\Delta=2g$. All policies receive the initial vector and the
family description. They are not told whether or when a change occurs.

For a policy $\pi$ define
\[
 r_0(\pi)=\E_{\nu^0_g}[R_T],\qquad
 r_1(\pi)=\max_{a,\tau}\E_{\nu^{a,\tau}_g}[R_T].
\]
The constrained value, for $r\in[0,gT]$, is
\begin{equation}
 \mathfrak F_g(r)=\inf_{\pi:\,r_0(\pi)\le r}r_1(\pi).
 \label{eq:frontier-def}
\end{equation}
Knowing the family is a restricted promise. The corresponding upper
bound is not a general non-stationary bandit algorithm.

\begin{theorem}[Stationary-to-switching adaptation inequality]
\label{thm:adapt}
Let $n=A-1$ and $q=T-1$. If $r_0\le gnq/2$, then
\begin{equation}
 r_1\ge\frac{g^2nq}{8(r_0+gn)}.
 \label{eq:tradeoff}
\end{equation}
For every policy, $r_1\ge r_0-g$.
\end{theorem}
\begin{proof}
On the baseline, the expected total number of challenger pulls is
$r_0/g$. Thus some challenger has expected pull count in rounds
$2,\ldots,T$ at most $r_0/(gn)\le q/2$. The gap-coupling calculation
in Lemma~\ref{lem:mixture}, using its fixed-arm bound and averaged
over the $q$ change times, gives
\[
 r_1\ge\frac{gq}{8(r_0/(gn)+1)}
       =\frac{g^2nq}{8(r_0+gn)}.
\]
The challenger is chosen by its expected baseline count, not as a
function of the realized private random tape.
For the second inequality, choose any challenger and $\tau=T$.
The histories and actions through round $T-1$ coincide with the
baseline. At round $T$, changing that arm can reduce instantaneous
regret by at most $g$. Therefore its expected regret is at least
$r_0-g$.
\end{proof}

\begin{theorem}[Matching frontier on the disclosed one-change family]
\label{thm:frontier}
For $T\ge16A$, $g\in(0,1/2]$, and $r\in[0,gT]$, universal
constants $c,C>0$ satisfy
\begin{equation}
 c\Phi_g(r)\le\mathfrak F_g(r)\le C\Phi_g(r),
 \qquad
 \Phi_g(r)=\min\left\{gT,\;
 g\sqrt{nT}+\frac{g^2nT}{r+gn}\right\},
 \label{eq:frontier}
\end{equation}
where $n=A-1$. Thus the lower-bound obstruction is attainable, up to
constants, throughout the stationary-budget range on this family.
\end{theorem}

\begin{proof}[Lower bound]
Set $u=\sqrt{nq}$, so $u\ge4$ and $n\le u$. If
$r_0\ge gu/2$, Theorem~\ref{thm:adapt}'s second inequality gives
$r_1\ge gu/4$. Otherwise its first inequality applies and gives
\[
 r_1\ge\frac{g^2nq}{8(r_0+gn)}\ge\frac{gu}{12}.
\]
Consequently every policy has $r_1\ge c_0g\sqrt{nT}$, independently
of its stationary budget.

For a policy feasible in \eqref{eq:frontier-def}, if
$r\le gnq/2$, \eqref{eq:tradeoff} also gives
$r_1\ge g^2nq/(8(r+gn))$. If $r>gnq/2$, then
\[
 \frac{g^2nT}{r+gn}\le\frac{2gT}{q}\le4g,
\]
and this term is already bounded by a constant times the established
floor. Combining the two lower bounds, and using $q\ge T/2$,
proves a constant times \eqref{eq:frontier} for every feasible policy.
\end{proof}

\begin{algorithm}[tb]
\caption{Independent monitoring with a stationary-regret budget}
\label{alg:frontier}
\begin{algorithmic}[1]
\Require The disclosed one-change family, $A,T,g$, stationary budget $r$
\State Set $n=A-1$ and $p=\min\{r/(gT),\sqrt{n/T}\}$.
\For{$t=1,\ldots,T$}
 \If{an improved challenger has already been observed}
  \State Pull that challenger.
 \Else
  \State With probability $p$, pull a uniformly random challenger;
  otherwise pull arm $1$.
  \State Observe its reward. If a challenger returns $1/2+g$, record it
  as improved.
 \EndIf
\EndFor
\end{algorithmic}
\end{algorithm}

\begin{proof}[Upper bound]
On the baseline, Algorithm~\ref{alg:frontier} never detects an
improvement, and has $r_0=gpT\le r$. For an alternative, while no
improvement has been observed, its challenger is probed independently
with probability $p/n$. From the change time, the expected time to
its first probe is at most $n/p$, capped by the remaining horizon.

Before detection, the instantaneous regret is at most the baseline
probe cost $g\ind\{a_t\ne1\}$ plus an
additional $g$ in a post-change round that does not pull the improved
arm. After detection, regret is zero. Hence, for $p>0$,
\begin{equation}
 r_1\le gpT+g\min\{T,n/p\}.
 \label{eq:frontier-upper}
\end{equation}
Every round's regret is also at most $2g$. If
$r\ge g\sqrt{nT}$, take $p=\sqrt{n/T}$ to obtain
$r_1\le2g\sqrt{nT}$. If $gn\le r<g\sqrt{nT}$, then
\[
 r_1\le r+\frac{g^2nT}{r}
 \le g\sqrt{nT}+\frac{2g^2nT}{r+gn}.
\]
The second step uses the two inequalities separately:
$r\le g\sqrt{nT}$ and $r+gn\le2r$, the latter following from
$r\ge gn>0$. In particular, no division by $r$ is used when $r=0$.
Together with $r_1\le2gT$, this is a constant times
$\Phi_g(r)$. Finally, if $0\le r<gn$, the fraction in
$\Phi_g(r)$ is at least $gT/2$, and the trivial $2gT$ bound
suffices. This includes $r=0$, when $p=0$ and the policy always
pulls the incumbent.
\end{proof}

\begin{corollary}[A variance-only stationary repair is insufficient]
\label{cor:impossible}
For fixed $A\ge2$ and fixed $g>0$, no policy receiving identical side
information on the baseline and alternatives can simultaneously achieve
polylogarithmic expected regret on the baseline and
$\softO(\sqrt T)$ expected regret on every one-change alternative.
\end{corollary}
\begin{proof}
Equation~\eqref{eq:tradeoff} gives
$\Omega(T/\operatorname{polylog}(T))$ alternative regret. The same
conclusion applies to uniform high-probability guarantees with
$\delta=T^{-2}$, since $0\le R_T\le T$ converts them into expected
bounds with additive cost at most $1/T$.
\end{proof}

\begin{corollary}[Oracle rates cannot be attained simultaneously]
\label{cor:nooracle}
No single policy without a supplied variation budget can achieve, for
all deterministic arrays, $\softO(\Psi(A,T,\Delta))$ expected regret.
\end{corollary}
\begin{proof}
For fixed $A$ and $g=1/4$, the baseline has $\Delta=0$ and oracle
rate $O(A)$, while every alternative has $\Delta=1/2$ and oracle
rate $O(A+\sqrt{AT})$. Corollary~\ref{cor:impossible} rules out
simultaneously attaining these rates up to logarithmic factors.
Disclosing the initial vector only makes the lower-bound problem easier.
\end{proof}

\paragraph{Interpreting the frontier.}
For constant $g$ and $A$, a stationary budget much smaller than
$\sqrt T$ causes a detection term of order $T/(r+1)$.
Allowing stationary regret of order $\sqrt T$ reaches the one-change
monitoring floor. This is a quantitative choice of guarantees, not a
claim that every stationary instance must always incur $\sqrt T$
regret. The unrestricted noisy model may admit richer instance-dependent
frontiers; the result here covers exactly the disclosed family.

\subsection{A polynomial loss relative to the supplied-budget oracle}
The preceding impossibility can be made quantitative without hiding
the dependence on $A$ and $T$. For a policy that receives no variation
budget, define its worst-case ratio to the rate function in
Theorem~\ref{thm:detminimax} by
\[
 \mathcal C(\pi;A,T)=\sup_{\mu\in[0,1]^{T\times A}}
 \frac{\E_\pi[R_T]}{\Psi(A,T,\Delta(\mu))}.
\]
The denominator is the supplied-budget benchmark, evaluated at the
actual variation. The policy knows the horizon and determinism but
receives identical side information on all arrays.
The matching upper frontier in Theorem~\ref{thm:frontier} uses its
family promise. In contrast, a lower bound on that family also constrains
a general-model policy: run the policy unchanged on those arrays,
ignoring any extra disclosure. The supremum defining $\mathcal C$
contains these arrays. Thus the following lower bound applies to the
full deterministic model, without asserting a matching general upper
frontier.

\begin{corollary}[Quantitative oracle-adaptation loss]
\label{cor:competitive}
If $T\ge16A$, every such policy satisfies
\begin{equation}
 \mathcal C(\pi;A,T)\ge c(T/A)^{1/4}
 \label{eq:competitive}
\end{equation}
for a universal $c>0$. This is a lower bound on the competitive
ratio; no matching unrestricted upper bound is claimed.
\end{corollary}
The policy quantifier is unrestricted: for every nonanticipating
randomized policy, including one designed for the full model, a
baseline or one-change array witnesses the obstruction. Only the
matching upper frontier uses the disclosed family promise.
\begin{proof}
Write $C_\pi=\mathcal C(\pi;A,T)$ and apply the disclosed family
with $g=1/4$, $n=A-1$ and $q=T-1$. Run $\pi$ unchanged,
ignoring the disclosed baseline vector and family promise.
Theorem~\ref{thm:adapt} applies because it holds for every policy
on that family, including those that ignore the extra information.
On the baseline, $r_0\le C_\pi A$. On every alternative,
\[
 r_1\le C_\pi(A+\tfrac12\sqrt{AT})
       \le\tfrac34 C_\pi\sqrt{AT}.
\]
The baseline and alternative hard mixture in
Lemma~\ref{lem:mixture} already implies $C_\pi\ge1/120$.
If $r_0>gnq/2$, then
$C_\pi>nq/(8A)\ge T/32$. For $T\ge16A$ and $A\ge2$,
$T/32\ge\tfrac12(T/A)^{1/4}$, proving the claim in this case.
Otherwise, Theorem~\ref{thm:adapt} gives
\[
 \tfrac34 C_\pi\sqrt{AT}
 \ge\frac{nq}{128(C_\pi A+n/4)}
 \ge\frac{T}{512(C_\pi+1/4)}.
\]
Hence $C_\pi(C_\pi+1/4)\ge\sqrt{T/A}/384$.
Using $C_\pi\ge1/120$ gives $C_\pi+1/4\le31C_\pi$,
and proves \eqref{eq:competitive}, for example with $c=1/128$.
\end{proof}

\subsection{Robustness to a finite improvement pulse}
The frontier is not restricted to an improvement that persists forever.
Let $\mathcal H_g^{\rm pulse}$ contain the baseline and arrays in
which one challenger has reward $1/2+g$ on $[\tau,\rho)$, where
$2\le\tau<\rho\le T+1$, and has its baseline reward outside
that interval. The initial vector and this single-pulse promise are
disclosed. A pulse ending before the horizon costs variation $4g$;
a persistent improvement costs $2g$.

\begin{corollary}[Single-pulse frontier]
\label{cor:pulse}
Define the constrained value as in \eqref{eq:frontier-def}, taking
the worst case over all pulses. Under the conditions of
Theorem~\ref{thm:frontier}, it is also $\Theta(\Phi_g(r))$.
\end{corollary}
\begin{proof}
Persistent alternatives ($\rho=T+1$) give the same lower bound.
For the upper bound, use Algorithm~\ref{alg:frontier} until the
improvement is observed, then pull the improved arm until its reward
returns to $1/2-g$. After that observation, play the incumbent for
all remaining rounds. Before detecting an improvement, the baseline
probing charge is at most $gpT$ in expectation. Undetected pulse
rounds contribute at most $g\min\{T,n/p\}$ in expectation for
$p>0$. The first round observing the return to baseline costs at
most one additional $g$. Thus
\[
 r_1^{\rm pulse}\le gpT+g\min\{T,n/p\}+g.
\]
The same cases as in Theorem~\ref{thm:frontier} apply; the added
$g$ is absorbed by $\Phi_g(r)\ge g$. The bound $2gT$ covers
$p=0$ and the small-$r$ case. The baseline regret remains $gpT$.
\end{proof}
The stopping rule uses the disclosed single-pulse promise. It provides
no guarantee for an unknown number of subsequent improvements.

\section{Implications for the general variance-aware problem}
\label{sec:adaptive}
The supplied-budget characterization and adaptation frontier impose
distinct constraints. An oracle minimax rate must recover the sharper
small-budget deterministic boundary. A single prior-free policy cannot
recover all oracle rates without an adaptation cost. That cost is already
nontrivial on a disclosed deterministic family, where estimation is free.

\begin{theorem}[Observable learner with supplied noisy budgets]
\label{thm:noisyoracle}
For supplied bounds $0\le Q\le T/4$ and $0\le B\le T-1$, there
is a policy knowing $A,T,\delta,Q,B$ but no individual variances,
means or change locations such that, on every environment with
$V_T\le Q$ and $\Delta\le B$, with probability at least $1-\delta$,
\begin{equation}
 R_T\le\min\left\{T,\;C D_T\left(
 A+\sqrt{AQ}+(AQBT)^{1/3}+\sqrt{ABT}\right)\right\}.
 \label{eq:noisyoracle}
\end{equation}
If $Q=0$, a separate deterministic branch attains
$C\Psi(A,T,B)$ pathwise. No knowledge of the actual variance of
a reward is required in either branch.
\end{theorem}
\begin{proof}
For $Q>0$, run the observable learner in
Algorithm~\ref{alg:sev} with the period of
Corollary~\ref{cor:oracle}; that corollary and $R_T\le T$
give \eqref{eq:noisyoracle}. If $Q=0$, every reward law is a point
mass. Apply Algorithm~\ref{alg:filtered} and
Theorem~\ref{thm:detminimax}. Both branches are selected using
supplied information, not an unobservable test for zero variance.
\end{proof}

Theorem~\ref{thm:noisyoracle} is a separate high-probability
implementation of the ordinary known-budget bound. The sharper
all-budget expected characterization is Theorem~\ref{thm:jointminimax}.
Its variance-sensitive cubic term uses $A Q$ with $Q$ an envelope
budget. In the earlier bound \citep{wei2016tracking}, the corresponding
term uses $A^2$ times a full-vector variance budget. The distinction
must be retained when comparing arm dependence. Proposition~\ref{prop:itoreduction} shows explicitly that the expected
ordinary restart rate follows from prior work. We retain this theorem for
its observable high-probability implementation and the minimax comparison.

\paragraph{Unknown-budget scope.}
No sharp all-budget guarantee is claimed for a policy that does not
receive $Q,B$. Our general-model upper bounds require these budgets;
the matching adaptive upper bound uses the disclosed deterministic
family. Removing both forms of side information while retaining
variance-sensitive regret is an open problem.

\subsection{An existing static guarantee already implies the expected upper rate}
\label{sec:itoreduction}
The noisy expected upper rate is not evidence of a new base algorithm.
Theorem~3, equation~(6), of \citet{ito2021hybrid} gives, against each
fixed arm $a$, a static expected-regret bound of the form
\begin{equation}
 \widetilde O\!\left(
 A+\sqrt{A\,\E\!\left[
 \sum_{t\in I}(\ell_{t,a_t}-u_{t,a_t})^2+
 \sum_{t,t+1\in I}\|u_{t+1}-u_t\|_1\right]}\right).
 \label{eq:itohybrid}
\end{equation}
The sequence $u_t$ is used only in the analysis, and is not supplied
to the algorithm. Rewards are converted to losses by $\ell_t=1-X_t$.

\begin{proposition}[Restarting a prior hybrid algorithm]
\label{prop:itoreduction}
Restarting the algorithm of \citet{ito2021hybrid} at fixed period $H$
gives
\begin{equation}
 \E[R_T]=\widetilde O\!\left(
 \sqrt{A\lceil T/H\rceil Q}+A\lceil T/H\rceil+HB\right)
 \label{eq:itoreduction}
\end{equation}
on $\mathcal E(A,T;Q,B)$. In particular, tuning $H$ with the
supplied budgets gives the same expected rate as
Theorem~\ref{thm:noisyoracle}.
Each block of length $h\ge\max\{A,3\}$ uses a fresh hybrid copy with horizon
$h$; on shorter blocks, play any arm. This defines the policy also
when a block contains a single round.
\end{proposition}
\begin{proof}
For a fixed interval $I$ of length $h\ge\max\{A,3\}$, set $u_t$ equal to the
mean loss vector $1-\mu_t$. The squared residual in
\eqref{eq:itohybrid} has expectation $\E[\Lambda_I]$, and the
path term is at most $A\Delta_I$. Converting static regret against
an arm optimal at the start of $I$ to dynamic regret costs at most
$2h\Delta_I$. Thus a fresh run has expected dynamic regret
\[
 \widetilde O\!\left(A+\sqrt{A\E[\Lambda_I]}
                         +A\sqrt{\Delta_I}+h\Delta_I\right).
\]
If $h\ge A$, then $A\sqrt{\Delta_I}\le A+h\Delta_I$.
For a shorter block, the trivial $R_I\le h\le A$ suffices. Therefore the
interval bound is
$\widetilde O(A+\sqrt{A\E[\Lambda_I]}+h\Delta_I)$.
Conditional on the history at a fixed block start, the oblivious
environment and a fresh algorithm seed justify the same bound.
Sum over blocks and use
$\sum_I\E[\Lambda_I]\le Q$, Cauchy--Schwarz and
$\sum_I h_I\Delta_I\le HB$ to prove \eqref{eq:itoreduction}.
The period in Corollary~\ref{cor:oracle} gives the last claim.
\end{proof}

This reduction is an inference from the cited theorem, not a dynamic
regret statement quoted verbatim from that paper. It does not by itself give the small-budget clipping refinement in
Section~\ref{sec:smallbudget}, the path-dependent high-probability theorem in
Appendix~\ref{sec:base}, or an unknown-budget scheduler. We therefore use ordinary restarting as a prior-work baseline for
the calibration and clipping refinement.

\section{Related work}
\label{sec:related}
Our comparison separates the comparator, feedback, variation
convention, variance convention and supplied information. Changing
any of these can invalidate a direct comparison of displayed rates.

Table~\ref{tab:priorrates} summarizes the closest finite-arm rate
comparisons. In its first row,
$V_0=\|1-\mu_1\|_\infty+\Delta$ is variation for losses with an
initial zero vector, and $W_T$ is the full-vector variance in
\eqref{eq:variances}. In particular, $V_0\le1+B$ and $W_T\le AQ$
on our class, but substituting these inequalities does not identify
the two budget conventions. The second row is a reduction carried
out here from a cited static theorem, not a dynamic theorem quoted
verbatim from that source.
\begin{table}[H]
\centering\small
\caption{Closest finite-arm guarantees. Upper rates suppress logarithms
and the trivial horizon cap. All rows compare to the best mean arm at
each round; the supplied information and budget conventions differ.}
\label{tab:priorrates}
\begin{tabularx}{\textwidth}{@{}p{0.23\textwidth}X p{0.28\textwidth}@{}}
\toprule
Source & Rate & Scope and attribution \\
\midrule
\citet{wei2016tracking}, Theorem 3.1 &
$(A^2W_TV_0T)^{1/3}+\sqrt{AV_0T}$ &
Supplied full-vector variance and initial-inclusive variation; prior result. \\
\addlinespace
\citet{ito2021hybrid}, Theorem 3, plus restarting &
$A+\sqrt{AQ}+(AQBT)^{1/3}+\sqrt{ABT}$ &
Supplied $Q,B$; direct reduction in Proposition~\ref{prop:itoreduction}. \\
\addlinespace
Theorem~\ref{thm:jointminimax} &
$A+\sqrt{AQ}+(AQBT)^{1/3}+\sqrt{AT}\min\{B,\sqrt B\}$ &
Supplied $Q,B$; matching lower bound up to logs. The noisy refinement is for $B<1$. \\
\bottomrule
\end{tabularx}
\end{table}

\paragraph{Variation budgets and the deterministic boundary.}
The stochastic variation-budget theory of
\citet{besbes2014,besbes2019} yields the familiar cubic rate.
MOSS-based resetting, sliding-window and discounted policies are
studied by \citet{weisrivastava2024}.
These stochastic models permit nonzero observation noise, so their minimax lower bounds
do not establish the minimax value of the deterministic subclass.
\citet{karnin2016} already studied adaptation to unknown variation,
including an intermediate noiseless construction. Its supplementary
Theorem~A.1 already gives a prior-free deterministic guarantee of
$O(\log T\,[\sqrt{AT}+A^{1/3}T^{2/3}\Delta^{1/3}])$.
Supplementary Appendix~D outlines the general-arm noisy extension but
explicitly defers its full analysis. Thus neither
``non-stationarity with zero noise'' nor ``unknown variation'' is a
new problem formulation. In \citet{wei2016tracking}, Theorem~3.1
uses full-vector statistical variance and variation including the
initial loss vector. Its deterministic specialization already
recovers the large-budget scale, up to logarithmic factors.

There is also a short derivation of the small-budget upper rate from
that earlier result. After the filter in Lemma~\ref{lem:filter},
every retained arm's future reward lies in
$[y_{\max}-3B,y_{\max}+B]$. If $0<B\le1$, transform these rewards
by $x\mapsto(x-y_{\max}+3B)/(4B)$. The transformed rewards are
in $[0,1]$, have consecutive variation at most $1/4$, and have
initial-vector variation at most $1$. Applying the deterministic
specialization of \citet{wei2016tracking} on the remaining horizon
and undoing the scaling gives $\softO(B\sqrt{AT})$, in addition
to the $A$ initialization cost. This is a derivation in this manuscript,
not a theorem attributed verbatim to that paper. Theorem~\ref{thm:detminimax}
provides an elementary logarithm-free upper bound and an explicit
oblivious lower bound throughout the budget range. The resulting deterministic scale is therefore compatible with the
earlier upper bound; the calibration and clipping argument is needed
to handle positive noise in the small-budget region.

\paragraph{Static path-length regret is a different guarantee.}
\citet{wei2018adaptive,bubeck2019path} obtain adaptive static-regret
bounds in terms of realized loss variation. In particular, Theorem~2 of \citet{bubeck2019path} gives
$\softO(A+\sqrt{A P_\infty})$ when an upper bound on $P_\infty$ is supplied, where
$P_\infty=\sum_{t=2}^T\|\ell_t-\ell_{t-1}\|_\infty$ for a
deterministic loss array. This compares to a fixed arm, whereas our
$R_T$ compares to a possibly different best arm every round.
For stochastic observations, realized path length and mean variation
also differ: independent Bernoulli$(1/2)$ rewards have $\Delta=0$,
but already one arm contributes expected path length $(T-1)/2$.
These static-comparator bounds do not directly imply a variance-sensitive
dynamic-regret guarantee without supplied budgets.

\paragraph{Prior-free tracking and significant shifts.}
ADSWITCH \citep{auer2019} adapts to distribution changes, and
the contextual replay approach of \citet{chen2019} simultaneously
controls switching and distributional total variation.
The latter's variation uses total variation distance between laws,
which need not be comparable from above by our mean variation.
MASTER \citep{wei2021master} supplies a broad reduction; its
Section~5 identifies difficulties with data-dependent base guarantees.
META \citep{suk2022} goes further by tracking significant arm shifts.
Its Theorem~1 and Corollaries~1--2 imply, on one and the same policy,
\[
 R_T=\softO\!\left(\min\left\{
 \sqrt{A(\widetilde S+1)T},\;
 \sqrt{AT}+(A\Delta)^{1/3}T^{2/3}\right\}\right),
\]
where $\widetilde S$ counts significant shifts and is no larger than
the number of best-arm switches. ArmSwitch \citep{abbasi2023}
independently adapts to best-arm switches and gives a variation bound
for a modified elimination condition. These are stronger comparators
than merely counting all mean changes. Their unit-range variation
term has the scale $(A\Delta)^{1/3}T^{2/3}$, corresponding to the
variance-sensitive cubic term $(AV_T\Delta T)^{1/3}$ when
$V_T=\Theta(T)$. Achieving low-noise dependence together with
adaptation to unknown budgets remains open.

\paragraph{Exploration--detection tradeoffs and alternative objectives.}
The inverse stationary-regret tradeoff is classical. We cite both the
published version \citep{garivier2011} and the detailed preprint
\citep{garivier2008}, whose Theorem~13 uses finite-KL stochastic
pulse alternatives. That theorem does not apply verbatim to disjoint
deterministic reward laws. Our coupling proof treats that boundary
and gives a matching policy on a disclosed family; it is not a claim
to have discovered the exploration--detection obstruction.
The polynomial consequence in Corollary~\ref{cor:competitive} also
follows from this obstruction and the oracle benchmark.
Theorem~31.2 of \citet{lattimore2020} gives a particularly direct
two-arm Gaussian pulse formulation of the inverse-regret tradeoff.
\citet{marinov2021pareto} study model-selection frontiers, with an
adaptive-adversary switching obstruction. \citet{qian2026} achieve
simultaneous static and dynamic bounds for deterministic losses,
an oblivious adversary and a supplied switch-count bound; their
stationary static guarantee has scale $\sqrt{AT}$.
Satisficing regret \citep{zhang2026satisficing} compares against a
threshold, rather than the best arm, and is a different objective.
These results do not contradict our small-stationary-regret tradeoff.

\paragraph{Variance estimation and detection frameworks.}
Variance-based stochastic and best-of-both-worlds methods include
\citet{audibert2009,ito2021hybrid,ito2022}. A static guarantee does not
alone identify the dynamic rate. Proposition~\ref{prop:itoreduction}
gives the explicit reduction for \citet{ito2021hybrid}, rather than
classifying it as unrelated merely because its comparator is fixed.
The empirical-variance/restarting template is also present in
\citet{wei2016tracking}, Supplement~B.1; retaining per-arm residual
sums before Cauchy--Schwarz refines the coarse aggregation there. In the linear setting,
\citet{wang2025variance} give variance-sensitive restart algorithms
and a bandits-over-bandits variant, as well as a variance-dependent
linear-bandit lower bound (Theorem~4.1 and Appendix~C). The drift norm, dimension and
tuning information differ from our finite-arm model.
This distinction is quantitative: their Corollary~5.4 is stated for
variance and parameter-variation budgets bounded below by constants,
and Remark~5.6 discusses the rate loss in the finite-arm specialization.
Thus that paper is a direct variance-sensitive comparison, including
its finite-arm case, rather than being excluded solely by its linear
formulation.
GLR-klUCB \citep{besson2022}, DAB \citep{huang2026dab} and
DAL \citep{gerogiannis2026dal} provide practical detection/restart
frameworks. Their stated performance theorems impose explicit
spacing or detector-data conditions. For example, DAL's
Theorem~4.8 assumes Assumption~4.6. The phrase ``prior free''
therefore does not by itself mean a uniform guarantee over arbitrary
drifting arrays. Our comparison uses the AISTATS 2025 version of
\citet{wang2025variance}, the ICML 2026 version of DAL, and v2
of \citet{qian2026} dated 17 February 2026.
Earlier multiscale detection \citep{ding2020multiscale} similarly
uses explicit detectability conditions (Assumption~4 and Theorem~3
in its finite-arm specialization).

\paragraph{Structural promises and other observation models.}
ADR-bandit \citep{komiyama2024global} preserves logarithmic stationary
regret and handles coordinated global changes. Its Definitions~19--20
and Theorem~21 require detectable changes mirrored across arms;
our hidden-improvement alternatives keep the incumbent unchanged.
Thus its simultaneous positive result is compatible with our lower
bound. Recent misspecification reductions \citep{hu2026misspecification}
address round-specific linear action sets and supplied parameter path
length, rather than a noise-adaptive unknown-budget finite-arm theorem.
For heavy-tailed observations, \citet{genalti2026} give a Catoni-style
detector and a piecewise-stationary regret analysis under
Assumption~4.7. Their moment budget and separation conditions differ
from our bounded-reward variance-envelope class. These distinctions
are part of the comparison, not grounds for omitting adjacent work.

Smoothness and gap structure provide further forms of side information.
\citet{slowly2025} use a supplied per-round drift limit;
\citet{manegueu2021generalized} analyze controlled changes of gaps
and the best mean, with tuning parameters supplied to their algorithm.
\citet{suk2025smooth} adapt to unknown H\"older smoothness and give
faster gap-dependent rates under a safe-environment condition.
In the public v2, Definition~5 and Theorem~7 impose the existence of
an arm whose cumulative gaps remain below a sampling threshold on
every interval. A hidden improvement with long pre- and post-change
segments can violate this condition, even at zero observation variance.
These gap-dependent positive results are therefore relevant structural
comparisons rather than uniform guarantees for our entire class.
Finally, \citet{chen2025bridging} interpolate between static and
dynamic oracles using a switch-budget-constrained arm sequence.
Their Algorithm~2 receives both switch and variation budgets;
its two-regime bound concerns that comparator and does not supply
unknown-budget low-noise adaptation.

\begin{proposition}[Sampling variance survives deterministic rewards]
\label{prop:iwvariance}
For an arm sampled with predictable probability $p>0$, let
$Z=\ind\{a_t=a\}X_{t,a}/p$. Conditional on the pre-round history,
\begin{equation}
 \Var(Z)=\frac{\sigma_{t,a}^2}{p}
          +\mu_{t,a}^2\left(\frac1p-1\right).
 \label{eq:iwvariance}
\end{equation}
More generally, with a predictable baseline $b$, replacing $X_{t,a}$
by $X_{t,a}-b$ replaces $\mu_{t,a}^2$ by $(\mu_{t,a}-b)^2$.
\end{proposition}
\begin{proof}
Here the history is taken before drawing the action's fresh private
randomness, and $X_{t,a}$ denotes a potential draw from $\nu_{t,a}$.
The conditional first moment of $Z$ is $\mu_{t,a}$ and its
second moment is $(\sigma_{t,a}^2+\mu_{t,a}^2)/p$.
Subtracting the square of the first moment proves the identity.
The centered case is the same calculation.
\end{proof}
This elementary identity explains a technical obstacle to directly
substituting a reward-variance bonus into an importance-weighted replay
test: zero reward noise does not imply zero variance of its estimator.
Centering can help, but controlling its residual and the sampling
schedule requires a separate argument.

\paragraph{Reward ranges, clipping and statistical variation.}
Range adaptation is an established issue. \citet{hadiji2023range}
show a tradeoff when the support interval of stationary rewards is
unknown (Theorem~3). Our original support $[0,1]$ is known;
calibration certifies a \emph{mean} interval using a supplied drift
budget, after which the learner chooses the clipped-feedback support.
This does not infer unknown reward-support endpoints.
\citet{gerchinovitz2016} show that static adversarial regret cannot
scale with the across-arm effective loss range alone (Corollary~4).
Here the total mean path is constrained and clipping bias is charged
to the variance envelope; a small instantaneous range alone would
not justify our guarantee. The quadratic-variation result of
\citet{bubeck2018sparsity}, Theorem~2, concerns realized loss vectors
and a fixed comparator. These conventions differ from observation
variance about a drifting mean. The hybrid theorem of
\citet{ito2021hybrid} is the explicit static ingredient in our upper
bound. Calibration, clipping and balancing its bias are the additional
steps in the supplied-small-budget argument, rather than a new
static-regret regularizer. A separate work of
\citet{ito2021semi} extends hybrid data-dependent guarantees to
combinatorial semi-bandits and adversarial linear bandits. Its static
comparison and realized loss-vector quadratic variation and path length
are distinct from our dynamic comparator and reward-noise envelope;
it is not the COLT theorem invoked in \eqref{eq:itohybridexplicit}.

Earlier detection algorithms also address robustification and
monitoring. CD-UCB \citep{liu2018cd} uses a change-detection wrapper;
its CUSUM guarantees impose segment-length, detectability and Bernoulli
conditions. M-UCB \citep{cao2019mucb}, Assumption~1 and Theorem~1,
requires a window-dependent separation and a detectable mean shift.
\citet{mukherjee2019} analyze gap-dependent bounds under global,
separated changes. \citet{han2024forced} study forced-exploration
schedules with a gap-dependent sub-Gaussian analysis.
RCD-UCB \citep{wang2021robust} already uses truncated loss in a
robust detector; its Assumption~2.1 and Theorem~2.2 retain separation
and detectability conditions. Thus truncation itself is not a new
bandit technique. Our clipping step bounds the regret of a restarted
learner by using original means as an analytical reference; it does
not attempt to detect individual change points.

\paragraph{Directly observed mean tracking.}
The Anytime Tracking CUSUM method of \citet{gafni2026cost} studies
multiple change points without minimum spacing or jump-size assumptions.
Its Section~2, however, provides a stream observation every round and
measures cumulative squared mean-prediction error, with a known
sub-Gaussian proxy. There is no choice of which arm to observe.
This is an important adjacent monitoring result, but its guarantee
cannot be substituted for a selected-arm dynamic-regret bound.
Its Theorem~4.1 also retains a mean-diameter contribution in addition
to its variance terms; we do not interpret it as a zero-cost guarantee
at zero noise.

\paragraph{Summary of the comparison.}
Section~\ref{sec:related} and Appendix~\ref{app:comparison} distinguish
reward-noise variance, realized loss variation, and the side information
available to each policy. The ordinary expected restart bound is a
consequence of the prior static hybrid theorem. The small-budget
refinement additionally uses calibration, clipping and bias control.

\section{Conclusion}
We give regret bounds across the full range of supplied variance and
mean-variation budgets. The deterministic minimax rate is matched up to
universal constants; for positive variance, the upper and lower bounds
differ by at most one logarithmic factor. Safe calibration and clipping
provide the small-budget refinement of the ordinary restart guarantee.
On a disclosed deterministic one-improvement family, we characterize
the stationary/change-regret frontier and its extension to one finite
pulse. The same hard instances imply a general lower bound on adaptation
to unknown variation budgets. Removing the logarithmic gap in the noisy
minimax bound, and obtaining matching adaptation frontiers for unrestricted
noisy environments, remain open.

\appendix
\section{Oblivious deterministic lower bounds}
\label{app:detlower}
The following constructions prove the lower bound in
Theorem~\ref{thm:detminimax}. The mixtures are fixed before interaction.

\subsection{A fixed hard distribution}
Our lower bound uses a distribution fixed before interaction. This
avoids selecting an epoch's environment after observing the learner's
realized history.

\begin{lemma}[One-block hidden-improvement mixture]
\label{lem:mixture}
Fix $h\ge A\ge2$ and $g\in(0,1/2]$. In a block of $h$ rounds,
use baseline rewards $1/2$ for arm $1$ and $1/2-g$ for every
challenger. With probability $1/2$ keep this baseline throughout.
Otherwise choose a challenger uniformly and a change time uniformly
from $\{2,\ldots,h\}$, and change that challenger to $1/2+g$
from that time onward. For every policy, even with the baseline
disclosed in advance and any initial private state independent of this
block's random draw,
\begin{equation}
 \E[R_{\rm block}]\ge\frac g{40}\sqrt{Ah}.
 \label{eq:mixture}
\end{equation}
Every supported array has internal variation at most $2g$.
\end{lemma}
\begin{proof}
Set $n=A-1$ and $q=h-1$. On the baseline, let $N_a$ count pulls of
challenger $a$ in rounds $2,\ldots,h$, and let
$N=\sum_{a\ge2}\E[N_a]$. Its expected regret is at least $gN$.
For a fixed challenger and a fixed policy seed, its baseline trajectory
has $N_a+1$ gaps containing no pull of that arm, of lengths $d_j$ with
$\sum_jd_j=q-N_a$. A change in a gap is unobserved until the next
pull, and the coupled actions remain the same. The sum of detection
delays over change positions in that gap is $d_j(d_j+1)/2$.
Each undetected round costs at least $g$.

Let $U$ be the learner's private random tape and initial private state,
independent of the challenger and change-time draw. Write $R^{a,t}_{\rm block}$
for regret on the alternative changing challenger $a$ at time $t$.
The function $f(x)=(q-x)^2/(x+1)$ satisfies, on $[0,q]$,
\[
 f'(x)=1-\frac{(q+1)^2}{(x+1)^2}\le0,\qquad
 f''(x)=\frac{2(q+1)^2}{(x+1)^3}>0.
\]
For a fixed challenger, Cauchy--Schwarz on its baseline gap lengths
and then Jensen over $U$ give
\[
 \frac1q\sum_{t=2}^h\E_U[R^{a,t}_{\rm block}]
 \ge\frac g{2q}\E_U[f(N_a(U))]
 \ge\frac g{2q}f(\E_U[N_a]).
\]
The first inequality drops the nonnegative linear part of
$\sum_jd_j(d_j+1)$. In particular, for any fixed arm with
$\E_U[N_a]\le v\le q/2$, monotonicity gives a bound of
$gq/[8(v+1)]$. Averaging instead over the independent uniform
challenger and applying Jensen again bounds average alternative regret by
\begin{equation}
 \frac g{2q}\,
 \frac{(q-N/n)^2}{N/n+1}.
 \label{eq:mixture-delay}
\end{equation}
Let $u=\sqrt{nq}$. If $N\ge u/4$, the baseline half of the mixture
already contributes $gu/8$. Otherwise $N/n<q/2$ and
\eqref{eq:mixture-delay} is at least
\[
 \frac{gq}{8(N/n+1)}
 =\frac{gnq}{8(N+n)}\ge\frac{gu}{10},
\]
because $n\le q$ and $N<u/4$. The alternative half contributes at
least $gu/20$. Finally $n\ge A/2$ and $q\ge h/2$ give
$u\ge\sqrt{Ah}/2$, proving \eqref{eq:mixture}.
\end{proof}

\begin{proof}[Lower bound in Theorem~\ref{thm:detminimax}]
First, a stationary array with one uniformly random arm having reward
$1$ and all others reward $0$ forces expected regret at least
$(\min\{A,T\})/4$. Under the uniform choice of the hidden best arm,
the probability that it has been tested by round $t\le A$ is at most
$t/A$: until it is found, every observed reward is zero and at most
$t$ distinct arms can have been tried. Pulling the optimal arm in round
$t$ implies that it has been tested by that round. Hence
\[
 \E[R_T]\ge\sum_{t=1}^{\min\{A,T\}}(1-t/A)
 \ge\min\{A,T\}/4.
\]
This family has zero variation.

If $T<4A$, this lower bound is already a constant fraction of the
horizon and suffices. Hence suppose $T\ge4A$. For $0<B\le1$,
use Lemma~\ref{lem:mixture} with $h=T$ and $g=B/2$. Its variation
is at most $B$ and its expected regret is at least
$B\sqrt{AT}/80$.

For $B\ge1$, let
\[
 m=\min\{\lfloor B\rfloor,\lfloor T/(2A)\rfloor\}.
\]
Partition the horizon into $m$ blocks whose lengths differ by at most
one, and draw the mixture in Lemma~\ref{lem:mixture} independently
in each block, with $g=1/4$. Every block starts from the same baseline.
Internal changes cost at most $1/2$ per block; resetting at a boundary
costs at most $1/2$. Total variation is therefore at most
$m-1/2\le B$. Every block has length at least $2A$.
Conditional on any previous history, its fresh draw is independent of
the learner's private state, so the lemma remains applicable.
The expected total regret is at least a constant times
\[
 \sum_{i=1}^m\sqrt{Ah_i}
 \ge m\sqrt{A\lfloor T/m\rfloor}
 \ge\sqrt{AmT/2}
 =\Omega\!\left(\min\{T,\sqrt{ABT}\}\right).
\]
Here $m\ge\tfrac12\min\{B,T/(2A)\}$. Combining the initial
exploration bound with the drift bound, using
$\max\{x,y\}\ge(x+y)/2$ and the horizon cap, proves the result.
The mixture is an oblivious distribution over admissible deterministic
arrays. Its Bayesian lower bound implies the required worst-case
lower bound for each policy.
\end{proof}

\section{Variance-controlled statistical lower bounds}
\label{app:noisylower}
This appendix proves the statistical terms in
Theorem~\ref{thm:jointminimax}; all supported environments satisfy
the supplied budgets separately.

\subsection{A variance-controlled hard block}
\begin{lemma}[Symmetric two-point hard block]
\label{lem:noisyblock}
Fix $h\ge A\ge2$, $s\in(0,1/2]$, and
\[
 0<\epsilon\le\min\{s/2,\;(s/4)\sqrt{A/h}\}.
\]
Choose an arm $J$ uniformly from $[A]$ before a block of $h$ rounds.
All rewards take values $1/2-s$ and $1/2+s$. The probability of
the upper value is $1/2+\epsilon/(2s)$ for $J$, and $1/2$ for
every other arm. Every policy, with any initial private state
independent of $J$, has Bayesian expected block regret at least
$\epsilon h/4$. Each arm's variance is at most $s^2$, and at least
one arm has variance exactly $s^2$ at every round.
\end{lemma}
\begin{proof}
Let $P_0$ be the transcript law with every upper-value probability
equal to $1/2$, and $P_j$ the law with arm $j$ improved. If $N_j$
counts pulls of $j$, the KL chain rule and the Bernoulli KL bound give
\[
 \operatorname{KL}(P_0,P_j)
 =\E_0[N_j]\operatorname{kl}\!\left(\tfrac12,
                              \tfrac12+\tfrac{\epsilon}{2s}\right)
 \le\frac{2\epsilon^2}{s^2}\E_0[N_j].
\]
The second probability lies in $[1/2,3/4]$. This justifies the displayed
bound, for example by
$\operatorname{kl}(p,q)\le(p-q)^2/[q(1-q)]$.
Pinsker's inequality, $0\le N_j\le h$, and Jensen's inequality yield
\[
 \frac1A\sum_j\E_j[N_j]
 \le\frac hA+\frac{h\epsilon}{sA}\sum_j\sqrt{\E_0[N_j]}
 \le\frac hA+\frac{h\epsilon}{s}\sqrt{\frac hA}
 \le\frac{3h}{4}.
\]
The expected regret under $P_j$ is $\epsilon(h-\E_j[N_j])$.
Averaging proves the claim. The baseline variance is $s^2$; the
improved variance is $s^2-\epsilon^2$. Independent initial
randomness can be included in the transcript without changing the
KL calculation.
\end{proof}

\subsection{The statistical lower bound}
\begin{proof}[Lower bound in Theorem~\ref{thm:jointminimax}]
The deterministic class is a subclass of $\mathcal E(A,T;Q,B)$.
Theorem~\ref{thm:detminimax} therefore supplies
$\Omega(\min\{T,A+D_B\})$.
For $A\le T$ and $Q>0$, apply Lemma~\ref{lem:noisyblock} once
with $h=T$, $s=\sqrt{Q/T}$ and
$\epsilon=(s/16)\sqrt{A/T}$. Its means are constant and its
variance envelope is exactly $Q$, so it supplies $\sqrt{AQ}/64$.
For $A>T$, the deterministic initialization lower bound already
controls the horizon cap.

It remains to establish the $C_B$ term when it exceeds each of
$A,S,D_B$; otherwise one of the previous lower bounds suffices.
If $T<16A$, initialization alone is a constant fraction of the
horizon, so assume $T\ge16A$. Necessarily $B,Q>0$. Set
\[
 x=\left(\frac{B^2T^2}{AQ}\right)^{1/3},\qquad
 m=\lfloor x\rfloor,\qquad s=\sqrt{Q/T},\qquad
 \epsilon=\frac{s}{16}\sqrt{\frac{Am}{T}}.
\]
From $C_B\ge S$ we have $x\ge1$, so $x/2\le m\le x$.
If $B\le1$, the condition $C_B\ge D_B=B\sqrt{AT}$ gives
$Q\ge B^2\sqrt{AT}$, so $x^3\le(T/A)^{3/2}$ and
$x\le\sqrt{T/A}\le T/(4A)$.
If $B\ge1$, the same condition gives $Q^2\ge ABT$ and hence
$x^3\le Q^3/A^3$, so $x\le Q/A\le T/(4A)$.
Thus, in both cases,
\[
 1\le m\le x\le T/(4A).
\]
Partition $[T]$ into $m$ fixed blocks with lengths $h_i$ differing
by at most one. Every $h_i\ge4A$, and $h_i\le2T/m$.
Independently choose a fresh best arm in each block and use the laws
in Lemma~\ref{lem:noisyblock} with the common $s,\epsilon$ above.
Indeed, $Am/T\le1/4$ gives $\epsilon\le s/32<s/2$, and
\[
 \frac{\epsilon}{(s/4)\sqrt{A/h_i}}
 =\frac14\sqrt{\frac{mh_i}{T}}
 \le\frac{\sqrt2}{4}<1.
\]
This verifies both conditions of Lemma~\ref{lem:noisyblock} for
every block, including unequal lengths.

The entire random array of reward laws is drawn before interaction.
In every supported environment an unimproved arm exists because
$A\ge2$, so the maximum variance in every round is $s^2$ and the
envelope is exactly $Ts^2=Q$. At a
block boundary the max-norm mean change is at most $\epsilon$,
so its total variation is at most
\[
 (m-1)\epsilon\le m\epsilon
 \le\frac{\sqrt{AQ}}{16T}\,x^{3/2}
 =\frac{\sqrt{AQ}}{16T}\frac{BT}{\sqrt{AQ}}
 =B/16\le B.
\]
Conditional on the past, the new best arm is independent of the
learner's state. Summing the block lower bounds therefore gives
\[
 \E[R_T]\ge\frac{\epsilon T}{4}
 =\frac1{64}\sqrt{AQm}
 \ge\frac1{64\sqrt2}(AQBT)^{1/3}.
\]
The Bayesian lower bound implies a worst-case lower bound for each
policy, even if the fixed block boundaries are disclosed. Combining
the three constructions and using the maximum to bound one quarter
of the sum proves $c\Xi$.
\end{proof}

\section{An observable variance-sensitive base learner}
\label{sec:base}
This appendix gives the noisy upper-bound ingredient used in the main text. Variance bonuses are motivated by UCB-V \citep{audibert2009} and empirical Bernstein bounds \citep{maurer2009}; the proofs below instead use martingale concentration \citep{freedman1975}. Fix a contiguous epoch
$I=[s,e]$. At the beginning of round $t$, let $n_{t,a}$ be the number
of observations of arm $a$ in $[s,t-1]$. For $n_{t,a}\ge1$, let
\begin{align}
 \wh\mu_{t,a}&=\frac1{n_{t,a}}
 \sum_{u=s}^{t-1}\ind\{a_u=a\}X_u,\\
 \wh v_{t,a}&=\frac1{n_{t,a}}
 \sum_{u=s}^{t-1}\ind\{a_u=a\}(X_u-\wh\mu_{t,a})^2,\\
 r_{t,a}&=16\sqrt{\frac{\wh v_{t,a}\ell}{n_{t,a}}}
       +\frac{256\ell}{n_{t,a}},\qquad
 U_{t,a}=\min\{1,\wh\mu_{t,a}+r_{t,a}\}.
 \label{eq:radius}
\end{align}
These constants simplify a self-contained proof and are not intended as
optimized practical tuning. When $n_{t,a}=0$, set $U_{t,a}=1$.
An unplayed arm is selected before taking a maximizer among played arms.
All scores use pre-round information, and $n=1$ is well defined.

\begin{algorithm}[tb]
\caption{Scheduled empirical-variance UCB (SEV-UCB)}
\label{alg:sev}
\begin{algorithmic}[1]
\Require Arms $[A]$, horizon $T$, confidence $\delta$, restart period
$H\in\{1,\ldots,T\}$
\State Compute $\ell$ as in \eqref{eq:logs}.
\For{$s=1,1+H,1+2H,\ldots$ with $s\le T$}
  \State Set $n_a=0$, $m_a=0$, $S_a=0$ for every arm $a$.
  \For{$t=s,\ldots,\min\{s+H-1,T\}$}
    \If{some arm has $n_a=0$}
      \State Select the smallest-index unplayed arm $a_t$.
    \Else
      \State Compute \eqref{eq:radius} using $\wh\mu_{t,a}=m_a$,
      $\wh v_{t,a}=S_a/n_a$.
      \State Select a maximizer of $U_{t,a}$, breaking ties by index.
    \EndIf
    \State Observe $X_t$; set $d\gets X_t-m_{a_t}$.
    \State $n_{a_t}\gets n_{a_t}+1$;
    $m_{a_t}\gets m_{a_t}+d/n_{a_t}$.
    \State $S_{a_t}\gets S_{a_t}+d(X_t-m_{a_t})$.
    \Comment{Welford update}
  \EndFor
\EndFor
\end{algorithmic}
\end{algorithm}
The learner uses $O(A)$ memory and $O(A)$ time per round. It requires
no means, variances, gaps, or change locations. Its schedule parameter
$H$ remains an input; it is not a fully prior-free dynamic-regret policy.

\subsection{A simultaneous martingale event}
Let $\varepsilon_t=X_t-\mu_{t,a_t}$. For every arm $a$ and interval
$J=[u,v]$, define
\[
 S_{J,a}=\sum_{t\in J}\ind\{a_t=a\}\varepsilon_t,\quad
 Q_{J,a}=\sum_{t\in J}\ind\{a_t=a\}\varepsilon_t^2,\quad
 B_{J,a}=\sum_{t\in J}\ind\{a_t=a\}\sigma_{t,a}^2.
\]
These are analytical quantities, not algorithmic inputs.

\begin{lemma}[Simultaneous concentration]
\label{lem:event}
For any fixed nonanticipating policy, with probability at least
$1-\delta$, simultaneously for every interval $J$ and arm $a$,
\begin{align}
 |S_{J,a}|&\le4\sqrt{B_{J,a}\ell}+4\ell,
 \label{eq:noisebound}\\
 Q_{J,a}&\le2B_{J,a}+\ell,\qquad
 B_{J,a}\le2Q_{J,a}+2\ell.
 \label{eq:squares}
\end{align}
On the same event,
\begin{equation}
 \left|\sum_{t=1}^T\varepsilon_t\right|
 \le4\sqrt{\Lambda_T\ell}+4\ell.
 \label{eq:realnoise}
\end{equation}
\end{lemma}
\begin{proof}
Fix a deterministic $J$ and $a$. The centered selected observations
form a martingale-difference sequence bounded by $1$, with predictable
quadratic variation $B_{J,a}\le T$. Freedman's inequality, followed by
peeling $B_{J,a}$ into $[0,1]$ and dyadic bins up to $T$, gives
\eqref{eq:noisebound} with failure probability at most
$2(1+\lceil\log_2 T\rceil)e^{-\ell}$. The constants also cover
$B_{J,a}\le1$ since $\ell\ge1$.

For completeness, if $Y\in[0,1]$ and $m=\E[Y\mid\cF]$, convexity
of the exponential yields
\[
 \E[e^Y\mid\cF]\le e^{(e-1)m},\qquad
 \E[e^{-Y}\mid\cF]\le e^{-(1-e^{-1})m}.
\]
Apply the resulting exponential supermartingales to
$Y_t=\ind\{a_t=a\}\varepsilon_t^2$. Markov's inequality gives
$Q_{J,a}\le(e-1)B_{J,a}+\ell\le2B_{J,a}+\ell$ and
$(1-e^{-1})B_{J,a}\le Q_{J,a}+\ell$, implying
\eqref{eq:squares}, with total failure probability at most $2e^{-\ell}$.
The same Freedman argument for all selected observations gives
\eqref{eq:realnoise}.

There are at most $AT^2$ arm-interval pairs. Taking a union bound, using
$1+\lceil\log_2T\rceil\le T$ for $T\ge2$, and substituting
\eqref{eq:logs} bounds the total failure probability by $\delta$.
\end{proof}
The event is uniform over interval endpoints for one executed policy.
It is not an event simultaneously covering all possible policies.
In particular, a data-dependent choice among unexecuted policies does
not follow from this lemma.

\subsection{Confidence under drift}
\begin{lemma}[Empirical confidence with an explicit drift bias]
\label{lem:confidence}
On the event in Lemma~\ref{lem:event}, for every epoch $I$ and
pre-round empirical statistic in that epoch with $n_{t,a}\ge1$,
\begin{equation}
 |\wh\mu_{t,a}-\mu_{t,a}|
 \le r_{t,a}+2\Delta_I.
 \label{eq:confidence}
\end{equation}
\end{lemma}
\begin{proof}
Write the $n=n_{t,a}$ observed rewards as $Y_j=\theta_j+\eta_j$,
where $\theta_j$ is the mean of arm $a$ at its $j$-th sampled time.
Let $\bar\theta=n^{-1}\sum_j\theta_j$, $\bar Y=\wh\mu_{t,a}$,
$S=\sum_j\eta_j$, $Q=\sum_j\eta_j^2$, and
$B=\sum_j\Var(\eta_j\mid\text{past})$. Set $v=\Delta_I$ and
$F=\sum_j(Y_j-\bar Y)^2=n\wh v_{t,a}$.
All the means $\theta_j$ and the current mean lie within distance $v$
of one another. By expanding around $\bar Y$ and using
$\sum_j(\bar\theta-\theta_j)=0$,
\begin{equation}
 Q\le2F+\frac{2S^2}{n}+2nv^2.
 \label{eq:QF}
\end{equation}
If $n<256\ell$, the radius is at least $1$, so the result is trivial.
Otherwise, \eqref{eq:noisebound} gives
$S^2\le32B\ell+32\ell^2$. Together with $B\le2Q+2\ell$ and
\eqref{eq:QF}, this yields
\[
 B\le4F+4nv^2+\frac{128B\ell}{n}
       +\frac{128\ell^2}{n}+2\ell
 \le4F+4nv^2+\tfrac12 B+\tfrac52\ell.
\]
Hence $B\le8F+8nv^2+5\ell$. Returning to the noise bound,
\[
 \frac{|S|}{n}
 \le12\sqrt{\frac{\wh v_{t,a}\ell}{n}}
      +12v\sqrt{\frac\ell n}+\frac{13\ell}{n}
 \le r_{t,a}+v,
\]
because $n\ge256\ell$. Finally,
$|\bar\theta-\mu_{t,a}|\le v$ proves \eqref{eq:confidence}.
\end{proof}
The drift term in \eqref{eq:confidence} is only used in the proof. It
does not enter the implemented radius. This is different from asserting
that the empirical mean estimates the current mean without bias.

\subsection{A radius bound independent of the regret analysis}
\begin{lemma}[Radius control under drift]
\label{lem:radiuscontrol}
On the event in Lemma~\ref{lem:event}, take any executed epoch and
any arm with $n\ge1$ past observations. Let $B_{a,n}$ be the sum of
its conditional reward variances at those sampled times. If their
means have diameter at most $v$, the radius in \eqref{eq:radius} obeys
\begin{equation}
 r_{a,n}\le\frac{32\sqrt{B_{a,n}\ell}}n
          +\frac{280\ell}n+24v\sqrt{\frac\ell n}.
 \label{eq:radiusupper}
\end{equation}
This holds for every nonanticipating sampling rule, irrespective of
whether it maximizes a UCB score. In calibration one may take $v=B$;
in an epoch $I$ one may take $v=\Delta_I$.
\end{lemma}
\begin{proof}
Write the observations as $Y_j=\theta_j+\eta_j$ and set
$\bar\theta=n^{-1}\sum_j\theta_j$ and $Q_n=\sum_j\eta_j^2$.
The empirical mean minimizes squared deviations, so
\[
 n\widehat v_{a,n}
 \le\sum_j(Y_j-\bar\theta)^2
 \le2Q_n+2\sum_j(\theta_j-\bar\theta)^2
 \le2Q_n+2nv^2.
\]
Lemma~\ref{lem:event} gives $Q_n\le2B_{a,n}+\ell$. Hence
$\widehat v_{a,n}\le4B_{a,n}/n+2\ell/n+2v^2$.
Insert this into \eqref{eq:radius}, use
$\sqrt{x+y+z}\le\sqrt x+\sqrt y+\sqrt z$, and note that
$256+16\sqrt2\le280$ and $16\sqrt2\le24$.
This proves \eqref{eq:radiusupper} using concentration alone.
\end{proof}

\subsection{Interval regret}
For an epoch $I$, let
$\Lambda_I=\sum_{t\in I}\sigma_{t,a_t}^2$.
\begin{theorem}[Variance-sensitive interval guarantee]
\label{thm:interval}
On the event in Lemma~\ref{lem:event}, a fresh run of the base learner
on any executed epoch $I$ satisfies
\begin{equation}
 \Reg_I\le1000D_T\bigl(\sqrt{A\Lambda_I}+A+|I|\Delta_I\bigr).
 \label{eq:interval}
\end{equation}
The same statement holds for epochs produced by predictable restart
decisions, provided the learner is reset at every epoch start.
\end{theorem}
\begin{proof}
Let $h=|I|$ and $v=\Delta_I$. There are at most $A$ initial pulls of
unplayed arms, each with regret at most $1$. After initialization,
Lemma~\ref{lem:confidence} and clipping imply
\[
 U_{t,a}\ge\mu_{t,a}-2v,\qquad
 U_{t,a}\le\mu_{t,a}+2r_{t,a}+2v.
\]
Since the policy maximizes the scores, its instantaneous pseudo-regret
is at most $2r_{t,a_t}+4v$.

Lemma~\ref{lem:radiuscontrol} applies to these empirical statistics.
Let $N_a$ be the final epoch pull count and $B_a$ its final accumulated
variance. Summing over $n=1,\ldots,N_a-1$, using
$B_{a,n}\le B_a$, the harmonic-sum bound, and
$\sum_{n=1}^{N_a}n^{-1/2}\le2\sqrt{N_a}$, gives
\begin{align*}
 \sum_{t:\,n_{t,a_t}\ge1}r_{t,a_t}
 &\le32(1+\log h)\sqrt\ell\sum_a\sqrt{B_a}
    +280A\ell(1+\log h)+48v\sqrt\ell\sum_a\sqrt{N_a}\\
 &\le32(1+\log h)\sqrt{A\Lambda_I\ell}
    +280A\ell(1+\log h)+48v\sqrt{Ah\ell}.
\end{align*}
If $h<A$, every round is an initialization round and $\Reg_I\le A$.
Otherwise $\sqrt{Ah}\le h$. Combining the displayed estimates,
the initialization cost, and $4hv$ proves \eqref{eq:interval}, since
$\ell\ge1$ and $1+\log h\le1+\log T$.
The simultaneous interval event covers the realized endpoints of
predictably generated epochs as well.
\end{proof}

\section{Scheduled restarts and oracle tuning}
\label{sec:schedules}
\begin{theorem}[Partition guarantee]
\label{thm:partition}
For an executed partition $\cP$ into $M$ epochs on which the base
learner is reset, with probability at least $1-\delta$,
\begin{equation}
 \Reg_T\le C D_T\left(
 \sqrt{AM\Lambda_T}+AM+
 \sum_{I\in\cP}|I|\Delta_I\right),
 \label{eq:partition}
\end{equation}
where $C$ is a universal constant. For the fixed-period policy in
Algorithm~\ref{alg:sev}, $M=\lceil T/H\rceil$ and
\begin{equation}
 \Reg_T\le C D_T\left(
 \sqrt{A\lceil T/H\rceil\Lambda_T}
 +A\lceil T/H\rceil+H\Delta\right).
 \label{eq:fixedH}
\end{equation}
For realized regret, add $4\sqrt{\Lambda_T\ell}+4\ell$ to either
right-hand side. Pseudo-regret may always be capped by $T$.
\end{theorem}
\begin{proof}
Sum \eqref{eq:interval} over epochs. Cauchy--Schwarz gives
$\sum_I\sqrt{A\Lambda_I}\le\sqrt{AM\Lambda_T}$.
For the fixed schedule, every epoch has length at most $H$; use
\eqref{eq:boundary} to bound the drift contribution by $H\Delta$.
Equation~\eqref{eq:realnoise} and \eqref{eq:real} give the realized
regret statement.
\end{proof}

The partition guarantee is an analysis tool for a partition that was
actually executed. It does not assert that a learner can retrospectively
achieve the infimum over all partitions using the corresponding
counterfactual variances.

\begin{corollary}[Oracle variance--variation benchmark]
\label{cor:oracle}
Suppose deterministic budgets $B,Q$ are supplied, with
$\Delta\le B$ and $V_T\le Q$, where $0\le B\le T$ and
$0\le Q\le T/4$. If $B=0$, set $H=T$. If $B>0$, set
\begin{equation}
 H=\min\left\{T,\left\lceil\max\left\{
 1,\left(\frac{AQT}{B^2}\right)^{1/3},
 \sqrt{\frac{AT}{B}}\right\}\right\rceil\right\}.
 \label{eq:oracleH}
\end{equation}
With probability at least $1-\delta$, this policy satisfies
\begin{equation}
 \Reg_T\le C D_T\left(
 \sqrt{AQ}+A+(AQBT)^{1/3}+\sqrt{ABT}\right).
 \label{eq:oracle}
\end{equation}
No per-round variance is supplied to the base learner.
\end{corollary}
\begin{proof}
Substitute $\Lambda_T\le Q$ and $\Delta\le B$ into
\eqref{eq:fixedH}, and use $\lceil T/H\rceil\le T/H+1$ to reduce the
bound, up to a universal constant, to
\[
 \sqrt{AQ}+A+\sqrt{\frac{AQT}{H}}+\frac{AT}{H}+HB.
\]
When $B=0$, the assertion follows with $H=T$. Otherwise write
$h_v=(AQT/B^2)^{1/3}$ and $h_m=\sqrt{AT/B}$. If neither exceeds
$T$, the chosen $H$ is at least $h_v,h_m$, and at most
$h_v+h_m+2$. The two estimation terms are therefore at most
$(AQBT)^{1/3}$ and $\sqrt{ABT}$, respectively, and
$HB\le(AQBT)^{1/3}+\sqrt{ABT}+2B$.
Since $B\le T$ and $A\ge2$, $B\le\sqrt{ABT}$.
If $h_v>T$, then $BT<\sqrt{AQ}$; if $h_m>T$, then $BT<A$.
In either case setting $H=T$ is covered by the two stationary terms.
This also handles $Q=0$ and integer rounding.
\end{proof}

\paragraph{Limiting regimes.}
If $B=0$, the bound is $\softO(\sqrt{AQ}+A)$. If $Q=0$, it is
$\softO(A+\sqrt{ABT})$. For $Q=\Theta(T)$ it contains the usual
variation term $A^{1/3}B^{1/3}T^{2/3}$ together with the stationary
term. In the sublinear regime $AB\le T$, the monitoring term
$\sqrt{ABT}$ is no larger than this variation term, up to constants.
There is no contradiction with Theorem~\ref{thm:adapt}: here the
learner is told the budgets, and can change its schedule when the supplied
variation budget changes.

\section{Aggregation under explicit epoch charges}
The following lemma isolates the purely algebraic part of an epoch
analysis under explicit local variation charges for completed epochs.

\begin{lemma}[Aggregation from completed-epoch charges]
\label{lem:aggregation}
Let $h_i>0$, $\lambda_i\ge0$, and $v_i\ge0$ for completed epochs
$i=1,\ldots,m$, and suppose
\[
 \sum_i h_i\le T,\quad \sum_i\lambda_i\le\Lambda,\quad
 \sum_i v_i\le\Delta,\qquad
 h_i v_i\ge c(\sqrt{A\lambda_i}+A),
\]
for $c>0$. Then
\begin{equation}
 \sum_{i=1}^m(\sqrt{A\lambda_i}+A)
 \le c^{-1/3}(A\Lambda T\Delta)^{1/3}
       +c^{-1/2}\sqrt{AT\Delta}.
 \label{eq:aggregation}
\end{equation}
\end{lemma}
\begin{proof}
For the variance part,
\[
 \sqrt{A\lambda_i}
 =(A\lambda_i)^{1/3}(\sqrt{A\lambda_i})^{1/3}
 \le c^{-1/3}(A\lambda_i)^{1/3}h_i^{1/3}v_i^{1/3}.
\]
H\"older's inequality with three exponents equal to $3$ bounds its
sum by $c^{-1/3}(A\Lambda T\Delta)^{1/3}$. For the initialization
part, $h_i v_i\ge cA$ implies
$A\le\sqrt{A/c}\sqrt{h_i v_i}$. Cauchy--Schwarz then gives
$Am\le c^{-1/2}\sqrt{AT\Delta}$.
\end{proof}

\begin{remark}[A certificate is not a causal construction]
Lemma~\ref{lem:aggregation} proves no statistical claim. It concerns
completed epochs only, and says nothing about the last unfinished epoch.
An algorithmic theorem must establish the stated charges from observed
data, control drift regret as well, and bound the unfinished epoch.
If all these steps implied a polylogarithmic stationary zero-variance
cost and $\softO(\sqrt{AT})$ one-change cost for the same prior-free
policy, Corollary~\ref{cor:impossible} would rule them out.
Therefore a missing monitoring term cannot be repaired solely by this
aggregation calculation.
\end{remark}

\section{Stopped calibration and unconditional expected regret}
\label{app:stopping}
Fix an environment $\nu\in\mathcal E(A,T;Q,B)$. Let
$\mathcal H_s$ contain the actions, observed rewards and private
randomness actually used through round $s$. The calibration stopping
time $\tau$ is adapted to this filtration. Unused randomness, including
the seeds of future hybrid copies, is independent and is not included
in $\mathcal H_\tau$.

\paragraph{A prefix-measurable good event.}
Define $G$ by requiring \eqref{eq:noisebound} and \eqref{eq:squares}
only for arm--interval statistics contained in $[1,\tau]$.
For a fixed environment these inequalities are measurable from the
stopped history, even though their unknown means and variances make
$G$ an analytical event rather than an algorithmic test.
Thus $G\in\mathcal H_\tau$. The full simultaneous event in
Lemma~\ref{lem:event} implies $G$, so $\Prb_\nu(G^c)\le\delta$.
We use this implication only to bound the probability of $G^c$;
we never condition the future learner on the full simultaneous event.
The proof of Lemma~\ref{lem:calibration} uses only these prefix
inequalities. On $G$, its regret bound holds, and, if $\tau<T$,
both certificate properties hold: every future optimal arm is retained
and every surviving arm satisfies \eqref{eq:calibrationrange}.

\paragraph{Future observations after stopping.}
For each $s<T$, the event $G\cap\{\tau=s\}$ belongs to
$\mathcal H_s$. Independence across rounds implies that future
reward vectors retain their original laws conditional on this stopped
history, for almost every such history. No independence between arms
within a round is needed. In particular, for every retained arm and
$t>s$, the original conditional mean and variance are still
$\mu_{t,a},\sigma_{t,a}^2$. The clipping interval is fixed by that
history. Its distance certificate therefore gives the conditional
versions of both inequalities in \eqref{eq:clippingmoments}.
Fresh future seeds then justify applying the static hybrid guarantee
on each subsequent block. This proves the conditional assertion in
Lemma~\ref{lem:clipped} for the histories used by the upper bound.

\paragraph{The static theorem and the clipping conversion.}
For completeness, the proof of Theorem~3 in
\citet{ito2021hybrid}, printed page~12, gives, on $K$ arms and $h$
rounds, the explicit fixed-arm regret bound
\[
 6\sqrt{K\log h\;\E\!\left[
   \sum_{t=1}^h(\ell_{t,a_t}-u_{t,a_t})^2
   +K+8\sum_{t=1}^{h-1}\|u_{t+1}-u_t\|_1\right]}
 +3K\log h.
\]
The reference sequence is arbitrary and is not supplied to the algorithm
(Theorem~3, equation~(6), footnote~2). Thus it need not equal the
expectation of the observed loss. If $\mathcal Z_h(u)$ is as in
\eqref{eq:itohybridexplicit}, the quantity inside the expectation is
at most $8\mathcal Z_h(u)+K$. Since $\log h\ge1$ for $h\ge3$,
$\sqrt{x+y}\le\sqrt x+\sqrt y$ gives
\eqref{eq:itohybridexplicit}, for example with $C_0=18$.
The loss feedback is the selected coordinate only, as in the cited
multi-armed bandit model. On a retained set of size $K\ge2$ we have
$K\le A$ and invoke the algorithm only when $h\ge\max\{A,3\}$;
one survivor and shorter blocks are handled directly as stated in
Lemma~\ref{lem:clipped}.

Here is the conversion without suppressing the clipping biases.
Fix a good calibration history and a subsequent block-start history;
write $\E_s$ for expectation conditional on the latter. Set
$\bar Y_{t,a}=\E[Y_{t,a}\mid\mathcal H_s]$ and
$b_{t,a}=\bar Y_{t,a}-\mu_{t,a}$. Future reward vectors are independent
of the history before their round. Therefore
$|b_{t,a}|\le q_t/w$, and a nonanticipating action satisfies
$\E_s[Y_{t,a_t}]=\E_s[\bar Y_{t,a_t}]$.
For the fixed retained comparator $a^\circ$ optimal at the start of
this block, the exact identity is
\begin{align*}
 \E_s\sum_{t\in I}(\mu_{t,a^\circ}-\mu_{t,a_t})
 &=L\E_s\sum_{t\in I}(\ell'_{t,a_t}-\ell'_{t,a^\circ})
   +\E_s\sum_{t\in I}(b_{t,a_t}-b_{t,a^\circ})\\
 &\le L\E_s\sum_{t\in I}(\ell'_{t,a_t}-\ell'_{t,a^\circ})
      +\frac2w\sum_{t\in I}q_t.
\end{align*}
Biases can have either sign; the bound uses their absolute values.
Comparator losses are analytical coordinates of the reward vector,
not additional observations given to the learner. The reference
$u'_{t,a}=(z+5w-\mu_{t,a})/L$ lies in $[1/18,17/18]$, and
\[
 \E_s\sum_{t\in I}(\ell'_{t,a_t}-u'_{t,a_t})^2
 \le\frac{\E_s[\Lambda_I]}{L^2},\qquad
 \sum_{t,t+1\in I}\|u'_{t+1}-u'_t\|_1\le\frac{A\Delta_I}{L}.
\]
Substitution in the explicit static theorem gives the three terms
$LA$, $\sqrt{A\E_s[\Lambda_I]}$ and $A\sqrt{L\Delta_I}$, with
at most one logarithmic factor, as used in Lemma~\ref{lem:clipped}.
The additional dynamic-comparator cost is at most $2h\Delta_I$.
This establishes the conditional bridge before summing over blocks.

\paragraph{The actual post-calibration partition.}
Fix a stopped history with $\tau<T$. Consecutive blocks start at
$\tau+1$ and have length $H$, except possibly the last. Consequently,
\begin{equation}
 N_\tau=T-\tau,\qquad
 M_\tau=\left\lceil\frac{N_\tau}{H}\right\rceil
       \le M:=\left\lceil\frac TH\right\rceil.
 \label{eq:tailblocks}
\end{equation}
More explicitly, for $i=1,\ldots,M_\tau$, let
$I_i=[\tau+(i-1)H+1,\min\{\tau+iH,T\}]$ and $h_i=|I_i|$.
Set $q_t=\max_a\sigma_{t,a}^2$ and define the analytical tail quantities
\[
 Q_\tau=\sum_{t=\tau+1}^Tq_t\le Q,\qquad
 B_\tau=\sum_{t=\tau+2}^Td_t\le B.
\]
Conditional on the stopped history in a fixed environment, these
quantities and the partition are fixed. They are not new inputs to
the learner. Internal variation excludes changes at block boundaries,
so $\sum_i\Delta_{I_i}\le B_\tau$.

Let $\lambda_i=\E[\Lambda_{I_i}\mid\mathcal H_\tau]$.
The selected variance is bounded pathwise by the envelope on the
same rounds, giving $\sum_i\lambda_i\le Q_\tau$.
Conditioning the per-block bound at its start and then on
$\mathcal H_\tau$, with Jensen's inequality for the square root,
permits its variance term to be bounded by $\sqrt{A\lambda_i}$.
Thus the three block charges satisfy
\begin{align}
 \sum_i\sqrt{A\lambda_i}
 &\le\sqrt{AM_\tau\sum_i\lambda_i}
 \le\sqrt{AM_\tau Q_\tau},\nonumber\\
 \sum_i wA&=wAM_\tau,\qquad
 \sum_i h_i\Delta_{I_i}\le H B_\tau.
 \label{eq:tailcharges}
\end{align}
The factor $A$ in the square root comes from the static guarantee;
Cauchy--Schwarz is applied only over blocks and introduces
$\sqrt{M_\tau}$, not another arm factor. For a short block with
$h_i<\max\{A,3\}$, $h_i\le A$ and its direct regret bound
$16wh_i\le18wA$ uses the same single block charge. The two-round
case $A=2$ is included. The clipping-bias charges sum to $2Q_\tau/w$
over disjoint rounds; they are not charged once per block.

If $\tau=T$, take $M_\tau=Q_\tau=B_\tau=0$, with an empty
partition and zero post-calibration regret. This convention covers
the zero-tail endpoint without instantiating a learner.

\paragraph{The tower-property calculation.}
Write $R_T=R_{\rm cal}+R_{\rm post}$, setting $R_{\rm post}=0$
if $\tau=T$. The calibration regret is $\mathcal H_\tau$-measurable.
Let $K_{\rm cal}=C(A\ell+\sqrt{AQ\ell})$.
The preceding conditional sum and the explicit static ingredient give
the history-dependent bound
\[
 \mathcal K_\tau
 =C'\log(eT)\bigl(\sqrt{AM_\tau Q_\tau}
                +wAM_\tau+HB_\tau\bigr)+2Q_\tau/w.
\]
By \eqref{eq:tailblocks}, $Q_\tau\le Q$ and $B_\tau\le B$,
it is bounded by the uniform deterministic quantity
\[
 K_{\rm post}
 =C'\log(eT)\bigl(\sqrt{AMQ}+wAM+HB\bigr)+2Q/w
\]
with the following explicit constant convention. Let $C_{\rm long}$
be any universal coefficient for which the long-block bound is at most
\[
 C_{\rm long}\log(eT)
 \bigl(\sqrt{A\E[\Lambda_I\mid\mathcal H_{\min I-1}]}+wA+h\Delta_I\bigr)
 +\frac2w\sum_{t\in I}q_t.
\]
Such a coefficient follows from \eqref{eq:itohybridexplicit} and the
rescaling in Lemma~\ref{lem:clipped}. Fix once and for all
$C'\ge\max\{C_{\rm long},18\}$. Since $\log(eT)\ge1$, this same
coefficient covers every short-block charge $18wA$ as well; it is
independent of the budgets, the stopped history, and the block count.
Thus, on $G$,
$R_{\rm cal}\le K_{\rm cal}$ and
$\E[R_{\rm post}\mid\mathcal H_\tau]\le\mathcal K_\tau
\le K_{\rm post}$.
Using $0\le R_T\le T$ and $G\in\mathcal H_\tau$,
\begin{align}
 \E[R_T]
 &=\E[\ind_G R_{\rm cal}]
   +\E[\ind_G\E[R_{\rm post}\mid\mathcal H_\tau]]
   +\E[\ind_{G^c}R_T]\nonumber\\
 &\le K_{\rm cal}+\E[\ind_G\mathcal K_\tau]+T\delta\nonumber\\
 &\le K_{\rm cal}+K_{\rm post}+T\delta.
 \label{eq:stoppedexpectation}
\end{align}
This also covers horizon-truncated calibration. Taking $\delta=T^{-4}$
gives $T\delta=T^{-3}$, with no assumption that $G$ is independent
of the calibration output.

\paragraph{Proof dependencies.}
Lemma~\ref{lem:event} proves concentration.
Lemmas~\ref{lem:confidence} and \ref{lem:radiuscontrol} each follow
directly from it. These two lemmas prove calibration; the cited static
hybrid theorem and the two-part certificate prove clipped restarting.
Equation~\eqref{eq:stoppedexpectation} then combines them.
The separate interval-regret theorem, Theorem~\ref{thm:interval},
is not used to establish calibration or its radius estimate.

\clearpage
\section{Comparison with prior results}
\label{app:comparison}
Table~\ref{tab:comparison} summarizes the model conventions and
statements underlying the comparisons in Section~\ref{sec:related}.
The cited versions are specified where theorem numbering differs.
\begin{table}[H]
\centering\small
\caption{Models and guarantees in closely related work.}
\label{tab:comparison}
\begin{tabularx}{\textwidth}{@{}p{0.24\textwidth}p{0.24\textwidth}X@{}}
\toprule
Source and version & Relevant statement & Relation to the present bounds \\
\midrule
\citet{wei2016tracking}, NeurIPS 2016 &
Equation (1), Theorems 3.1 and 4.3; original supplement &
Variance--variation dynamic regret and noise-free monitoring are prior
themes. Full-vector variance and initial-inclusive variation differ
from $Q,B$. \\
\addlinespace
\citet{ito2021hybrid}, COLT 2021 &
Theorem 3, Eq. (6), and its explicit proof &
The arbitrary-reference static algorithm is prior work; ordinary
restarting is its direct consequence. Our added steps are calibration
and clipping-bias accounting. \\
\addlinespace
\citet{karnin2016}, NeurIPS 2016 supplement &
Theorem A.1 and Appendix D &
Prior-free variation tracking and a noiseless construction already
exist; the deterministic problem itself is not a new formulation. \\
\addlinespace
\citet{garivier2008,garivier2011}; \citet{lattimore2020} &
Preprint Theorem 13, Remark 15; textbook Theorem 31.2 &
The inverse stationary/change-regret obstruction is classical.
Our matching frontier is restricted to the disclosed deterministic family. \\
\addlinespace
\citet{wang2025variance}, AISTATS 2025 &
Theorems 4.1, 5.2; Corollary 5.4; Remark 5.6 &
Variance-sensitive linear results include a finite-arm comparison;
the stated theorem is not the same all-budget finite-arm formula. \\
\addlinespace
\citet{qian2026}, arXiv v2, 17 February 2026 &
Theorem 4 &
Simultaneous static/dynamic bounds use a supplied switch limit and
allow a stationary $\sqrt{AT}$ scale. This differs from our
small-stationary-regret constraint. \\
\addlinespace
\citet{gafni2026cost}, arXiv v2, 3 June 2026 &
Section 2; Theorems 4.1--4.2 &
Direct stream observations and squared mean-prediction error differ
from selected-arm feedback and best-arm pseudo-regret. The upper bound
also retains a mean-diameter term. \\
\bottomrule
\end{tabularx}
\end{table}
\bibliographystyle{alpha}
\bibliography{references}
\end{document}